%% file: paper.tex
\documentclass[twoside]{article}

\usepackage[accepted]{styles/aistats2023}
\runningauthor{V. Borovitskiy, M. R. Karimi, V. R. Somnath, A. Krause}

\usepackage{hyperref}
\usepackage{styles/VABcommands}[center]
\usepackage{styles/VABenvironments}
\usepackage[authoryearcomp]{styles/VABcitations}

\usepackage{cleveref}
\usepackage{url}
\usepackage{subcaption}
\usepackage{thm-restate}
\usepackage{booktabs}
\usepackage{multirow}
\usepackage{tikz}
\usepackage{standalone}

\newcommand{\pd}{\mathop{\dot{+}}}
\newcommand{\md}{\mathop{\dot{-}}}
\DeclareMathOperator{\Sym}{S}
\newcommand{\cayley}{\mathrm{Cayley}}

\bibliography{references.bib}

\begin{document}

\pagenumbering{arabic}

\twocolumn[

\aistatstitle{Isotropic Gaussian Processes on Finite Spaces of Graphs}

\aistatsauthor{ Viacheslav Borovitskiy\textsuperscript{\ensuremath{*1}} \And Mohammad Reza Karimi\textsuperscript{\ensuremath{*1}} \And Vignesh Ram Somnath\textsuperscript{\ensuremath{*1,2}} \And Andreas Krause\textsuperscript{\ensuremath{1}}}

\aistatsaddress{$^{1}$ Learning \& Adaptive Systems Group, Department of Computer Science, ETH Z\"urich, Switzerland
\\
$^{2}$IBM Research Z\"urich, Switzerland}]

\begin{abstract}
We propose a principled way to define Gaussian process priors on various sets of unweighted graphs: directed or undirected, with or without loops.
We endow each of these sets with a geometric structure, inducing the notions of closeness and symmetries, by turning them into a vertex set of an appropriate metagraph.
Building on this, we describe the class of priors that respect this structure and are analogous to the Euclidean isotropic processes, like squared exponential or Mat\'ern.
We propose an efficient computational technique for the ostensibly intractable problem of evaluating these priors' kernels, making such Gaussian processes usable within the usual toolboxes and downstream applications.
We go further to consider sets of equivalence classes of unweighted graphs and define the appropriate versions of priors thereon.
We prove a hardness result, showing that in this case, exact kernel computation cannot be performed efficiently.
However, we propose a simple Monte Carlo approximation for handling moderately sized cases.
Inspired by applications in chemistry, we illustrate the proposed techniques on a real molecular property prediction task in the small data regime.
\end{abstract}

\makeatletter
\fancyhead[CE]{\small\bfseries\@runningtitle}
\fancyhead[CO]{\small\bfseries\@runningauthor}
\makeatother

\section{Introduction}

Gaussian processes provide a principled framework to assess and quantify uncertainty, making them useful in various applications, e.g., in optimization \cite{snoek2012}, active \& reinforcement learning \cite{krause2007,deisenroth2011}.

Traditionally, Gaussian processes are applied to model functions $f: X \to \R$ where $X = \R^n$ is a Euclidean space.
However, many applications require modeling functions on different domains~$X$.
The main ingredient needed for this is defining a natural Gaussian process prior on $X$.
It should respect the geometric structure of $X$ and, at the same time, be fairly general-purpose in its nature.

In the Euclidean case, applications often rely on \emph{isotropic} priors, i.e., priors whose distribution is invariant with respect to translations and rotations, like squared exponential (RBF, Gaussian) or Mat\'ern Gaussian processes.\footnotemark
When conditioning such a prior by a translated and rotated dataset, the resulting model is transformed accordingly.

\begin{table}[b!]
\vspace*{-2.575ex}
\footnotesize\urlstyle{same}\textsuperscript{\ensuremath{*}}Equal contribution. Mail to: \href{mailto:viacheslav.borovitskiy@gmail.com}{viacheslav.borovitskiy@gmail.com} \\* Code available at: \url{https://github.com/vsomnath/graph_space_gps}. \\* Mirrored at \url{https://github.com/IBM/graph_space_gps}. 
\vspace*{3.425ex}
\end{table}

Extending this notion of isotropy to non-Euclidean $X$ has been a subject of recent work. 
For instance, 
\textcite{borovitskiy2020,borovitskiy2021} and \textcite{azangulov2022, azangulov2023} consider $X$ that is a Riemannian manifold or a vertex set of an undirected graph, where the notion of isotropy is substituted with invariance to Riemannian isometries or graph automorphisms, respectively.
Relying on Bochner's theorem-like constructions, they define appropriate priors and study (approximate) computational routines necessary for using them in various applications, e.g., in robotics \cite{jaquier2022} or wind speed modeling \cite{hutchinson2021}.

Following this principled paradigm, and motivated by applications in natural sciences, e.g. in chemistry, in this work, we consider domains $X$ that are sets of \mbox{unweighted} graphs on $n$ vertices (directed or undirected, with or without loops).
We endow these sets with appropriate geometric structure, turning them into \emph{spaces}, and \emph{derive} the generalized notion of isotropic Gaussian processes thereon.
We obtain, as special cases, the analogs of the landmark Mat\'ern and squared exponential Gaussian processes.
This, however, leaves us with ostensibly intractable kernels which we \emph{make tractable} by leveraging a finer structure of the setting.

\begin{figure*}
\begin{subfigure}{0.5\textwidth}
\centering
\includegraphics[height=9ex]{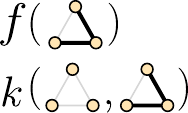}
\end{subfigure}
\begin{subfigure}{0.5\textwidth}
\centering
\includegraphics[height=10ex]{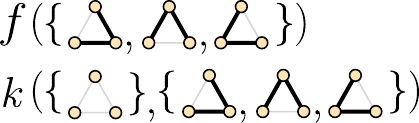}
\end{subfigure}
\caption{We study Gaussian processes $f$ and their respective covariance kernels $k$ in two settings.
First, $f$ that take \emph{graphs} as inputs (left).
Second, $f$ that take \emph{equivalence classes of graphs}, e.g., isomorphism classes of graphs, as inputs (right).}
\label{fig:setting}
\end{figure*}

We further consider a natural extension of the previous setting, whereby $X$ is now a set of equivalence classes of graphs under some permutation-induced equivalence relation, e.g., the set of graph isomorphism classes. One realistic use case of this arises when using graphs to encode molecules by associating nodes to atoms and edges to chemical bonds. Here, nodes corresponding to the same kind of atom are interchangeable while nodes corresponding to different types of atoms are not, calling for an equivalence class representation of a molecule, rather than for a graph representation.
We illustrate the two different settings we study in~\Cref{fig:setting}.

\footnotetext{Note: ARD versions \cite[page 106]{rasmussen2006} of Mat\'ern and squared exponential Gaussian processes are \emph{stationary} but not isotropic since they are not rotation invariant.}

To handle this setting, we propose \emph{projecting} the previously defined Gaussian process priors to make them into piecewise constant random functions on graph equivalence classes which we call $\emph{invariant versions}$.
We prove a hardness result that suggests that exact evaluation of kernels of the invariant versions cannot be done in any computationally efficient way, and suggest using straightforward Monte Carlo approximation to handle moderately sized problems.

To lend empirical support to our theoretical contributions, we evaluate the proposed methods on a real molecular property prediction task that mimics a typical application setting for Gaussian processes. We also consider a smaller subset of the same dataset where we can exactly evaluate the \emph{projected} kernels, i.e. the kernels of invariant versions, further demonstrating the utility of such geometrically structured priors when learning from limited data. 

\subsection{Gaussian Processes}

Gaussian processes \cite{rasmussen2006} are used as nonparametric probabilistic models for learning unknown functions.
A Gaussian process $f \sim \f{GP}(m, k)$ is a random function from some domain $X$ to reals.
Its distribution is determined by its \emph{mean function} $m(x) = \E f(x)$ and its \emph{covariance kernel} $k(x, x') = \Cov\del{f(x), f(x')}$, where the latter is necessarily \emph{positive semidefinite}.

Given a zero-mean Gaussian process prior $f \sim \f{GP}(0, k)$ and some data $x_1, y_1, \ldots, x_n, y_n$ with $x_i \in X, y_i \in \R$, one popular setting is to assume a Bayesian model $y_i = f(x_i) + \varepsilon_i$ with observations contaminated by i.i.d. noise $\varepsilon_i \sim \f{N}(0, \sigma^2_{\varepsilon})$.
If we denote $\v{y} = (y_1, \ldots, y_n)^{\top}$, this leads to posterior predictive $f \given \v{y} \sim \f{GP}(\hat{m}, \hat{k})$ given by pathwise conditioning \cite{wilson2020, wilson2021}
\[ \label{eqn:matheron}
f \given \v{y}\,(\cdot)
=
f(\cdot)
+
\m{K}_{\cdot \v{x}}
\del{\m{K}_{\v{x} \v{x}} + \sigma_{\varepsilon}^2 \m{I}}^{-1}
\del{\v{y} - f(\v{x})},
\]
where $\v{x}$ is defined similarly to $\v{y}$, $\m{I}$ is the identity matrix, $f(\v{x}) = (f(x_1), \ldots, f(x_n))^{\top}$,  and $\m{K}_{\v{a} \v{b}}$ is the matrix with elements $\del{\m{K}_{\v{a} \v{b}}}_{i j} = k(a_i, b_j)$.
Posterior moments $\hat{m}, \hat{k}$ are easily inferred to be
\[
\label{eqn:gp_posterior_moments_mean}
\hat{m}\,(\cdot)
&=
\m{K}_{\cdot \v{x}}
\del{\m{K}_{\v{x} \v{x}} + \sigma_{\varepsilon}^2 \m{I}}^{-1}
\v{y},
\\
\label{eqn:gp_posterior_moments_variance}
\hat{k}\,(\cdot, \cdot')
&=
k(\cdot, \cdot')
-
\m{K}_{\cdot \v{x}}
\del{\m{K}_{\v{x} \v{x}} + \sigma_{\varepsilon}^2 \m{I}}^{-1}
\m{K}_{\v{x} \cdot'}.
\]
The posterior mean function $\hat{m}(x)$ evaluated at $x \in X$ is the prediction at $x$, while the posterior standard deviation $\hat{k}(x, x)^{1/2}$ therein represents the respective uncertainty.

In order to use Gaussian processes in downstream applications, one needs a suitable prior and an inference algorithm.
For regression with Gaussian noise, the latter is given by~\Cref{eqn:matheron,eqn:gp_posterior_moments_mean,eqn:gp_posterior_moments_variance}.
In other settings, e.g., in classification, inference is not so straightforward, but is well-studied \cite{hensman2015, blei2017}.
Crucially, inference algorithms transfer to new domains in a straightforward way if kernel pointwise evaluation and prior sampling are available thereon.
We thus concentrate on building natural priors for Gaussian processes on various finite sets of graphs and their equivalence classes, which we proceed to discuss in the following section.

\subsection{Finite Spaces of Graphs and Equivalence Classes}
\label{sec:intro:input_spaces}

We study various sets of unweighted graphs on $n$ nodes.
Every such graph may be represented by its adjacency matrix or, after flattening to a vector, by an element of the set $\cbr{0, 1}^d$ with an appropriate~$d$.
For example, $d=n^2$ for the set of directed graphs with loops, which we denote by $\c{DL}_n$, where $\c{D}$ stands for \emph{directed} and $\c{L}$ for \emph{loops}.

We also consider the sets of undirected graphs with loops~(denoted by $\c{UL}_n$, where $\c{U}$ stands for \emph{undirected}), directed graphs without loops~(denoted by $\c{D}_n$) and undirected graphs without loops~(denoted by $\c{U}_n$).
Taking into account the structure of their adjacency matrices, these can be regarded as sets $\cbr{0, 1}^d$ with $d = n(n+1)/2$, ${d = n (n-1)}$ and $d = n(n-1)/2$, respectively.
We turn these sets into \emph{spaces} in~\Cref{sec:graph_gps}, by endowing them with an appropriate geometric structure.

We also consider sets of equivalence classes of graphs.
Extreme examples of these are isomorphism classes of graphs.
To define these, consider the group $\Sym_n$ of permutations of a size $n$ set.{\interfootnotelinepenalty=10000\footnote{For a discussion on the topic of groups we refer the reader to~\textcite{kondor2008} and \textcite{robinson2003}.}}
A permutation $\sigma \in \Sym_n$ acts on a graph~$x$ with $n$ nodes and adjacency matrix~$\m{A}_x$ returning the graph $y = \sigma(x)$ whose adjacency matrix $\m{A}_y$ has rows and columns permuted by $\sigma$.
Graphs $x$ and $y$ are called \emph{isomorphic} (denoted by ${x \isom y}$) if and only if there exists $\sigma \in \Sym_n$ such that $\sigma(x) = y$.
This defines an equivalence relation on any set $\c{V}$ of graphs and thus defines the equivalence classes (called \emph{graph isomorphism classes}): for $x \in \c{V}$ its $\isom$-equivalence class is $\overline{x} = \cbr{z \in \c{V} : z \isom x}$.
The set of such equivalence classes we denote by $\c{V}_{/\isom} = \cbr{\overline{x} : x \in \c{V}}$.

We also consider sets of equivalence classes defined by general equivalence relations determined by various subgroups $H \subseteq \Sym_n$.
Such a subgroup $H$ defines the equivalence relation $\sim_H$ on any set $\c{V}$ of graphs where $x \sim_H y$ if and only if there exists $\sigma \in H$ such that $\sigma(x) = y$.
Obviously, $\sim_{\Sym_n}$ is equal to the relation $\isom$ considered above.

\subsection{Previous Work and Contribution}

Defining a zero-mean Gaussian process prior amounts to defining a kernel. Kernels on various sets of graphs have been under consideration for a long time, see the recent surveys by \textcite{nikolentzos2021} and \textcite{kriege2020}.
The focus in previous work, however, is usually shifted towards other settings, most notably the one where node and edge features are more significant than the graphs' topology.
Moreover, the respective kernels are usually based on heuristics or aimed for a specific use case. To our best knowledge, principled general purpose kernels, like Mat\'ern kernels on graph spaces, have not been considered so far. 

Perhaps closest to our work is the piece in \textcite{kondor2002} about the hypercube diffusion kernel.
As we show later, this is the kernel of the squared exponential prior that we propose in~\Cref{sec:graph_gps}.
However, their kernel was studied without graph spaces in mind, and tractability was ensured using techniques that do not generalize to other isotropic kernels we study in~\Cref{sec:graph_gps}.

Kernels on graph isomorphism classes were studied before as well, see, e.g., \textcite{shervashidze2011}.
Notably, \textcite{gartner2003} prove that every strictly positive definite kernel of this sort cannot be computed exactly in an efficient way.
Our no-go results of~\Cref{sec:equiv_gps} are similar.

\emph{Our main contributions} are the following: (1) we propose a principled way of defining Gaussian process priors on finite graph spaces, endowing graph sets with geometric structure, thus making them into \emph{spaces}; (2) we make the respective ostensibly intractable kernels tractable; (3) we propose a principled way of projecting these priors on various sets of graph equivalence classes; (4) we prove a hardness result about the kernels between equivalence classes forbidding efficient \emph{exact} computation and (5) we demonstrate the performance of proposed methods on a real molecular property prediction task in the small data regime.

\section{Priors on Finite Spaces of Graphs}
\label{sec:graph_gps}

Let $\c{V}$ be any of the sets $\c{U}_n, \c{UL}_n, \c{D}_n, \c{DL}_n$ of unweighted graphs on $n$ nodes, identified with $\cbr{0,1}^d$ for the appropriate $d$ (see~\Cref{sec:intro:input_spaces} for definitions).
Note that $\abs{\c{V}} = 2^d$.

Defining a reasonable Gaussian process prior on $\c{V}$ requires endowing the set $\c{V}$ with some sort of "geometric" structure, which --- as a bare minimum --- defines a notion of closeness.
Arguably the most general way of encoding a geometric structure of a finite set is by making it into a vertex set of a graph.
Moreover, principled and practical Gaussian processes on nodes of finite graphs were studied before \cite{borovitskiy2021, kondor2002}, meaning that we can build upon the previous work.

\begin{figure*}
\begin{minipage}{0.029\textwidth}
\begin{subfigure}{\linewidth}
\centering
\includegraphics[width=\linewidth]{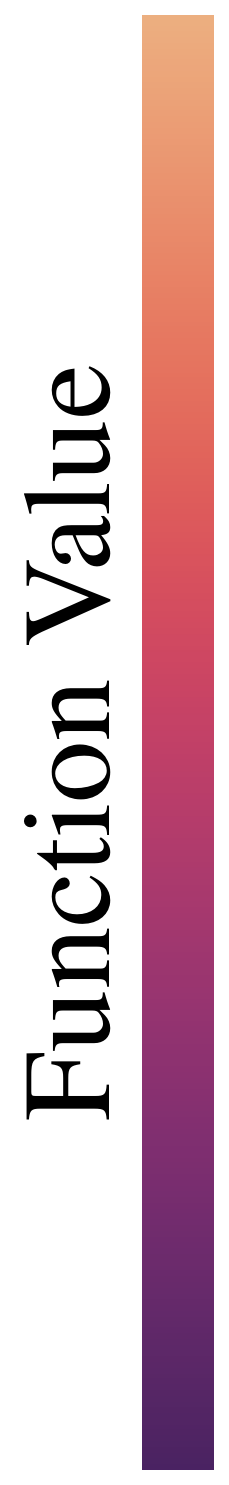}
\caption*{}
\label{fig:gpreg-colorbar-right}
\end{subfigure}
\end{minipage}
\begin{minipage}{0.93\textwidth}
\begin{subfigure}{0.24\textwidth}
\includegraphics[height=23ex]{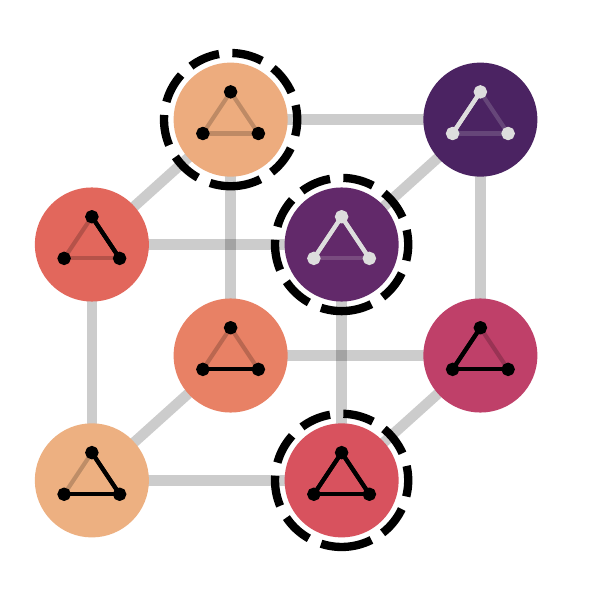}
\caption{Ground Truth}
\end{subfigure}
\hfill
\begin{subfigure}{0.24\textwidth}
\includegraphics[height=23ex]{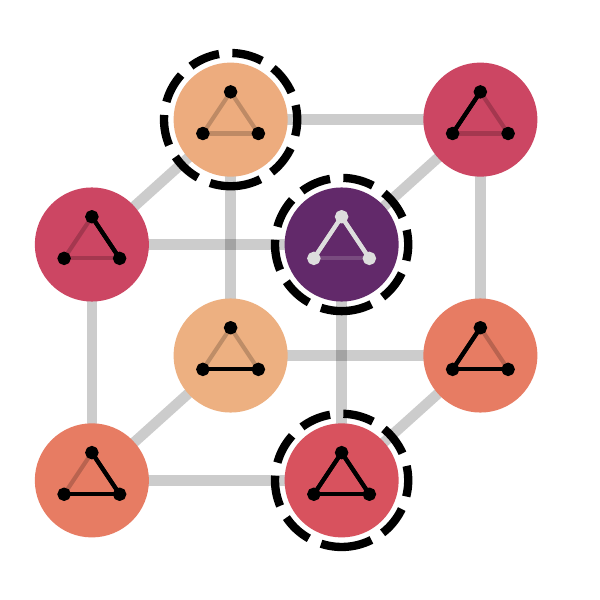}
\caption{Prediction}
\end{subfigure}
\hfill
\begin{subfigure}{0.24\textwidth}
\includegraphics[height=23ex]{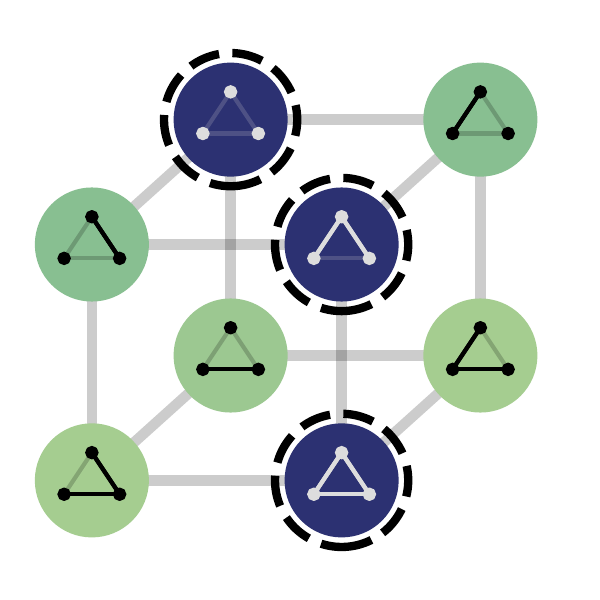}
\caption{Uncertainty}
\end{subfigure}
\hfill
\begin{subfigure}{0.24\textwidth}
\includegraphics[height=23ex]{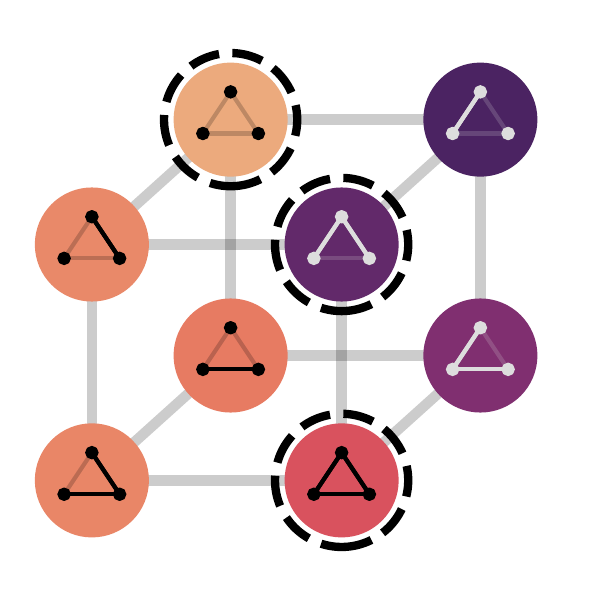}
\caption{Posterior sample}
\end{subfigure}
\end{minipage}
\begin{minipage}{0.029\textwidth}
\begin{subfigure}{\linewidth}
\centering
\includegraphics[width=\linewidth]{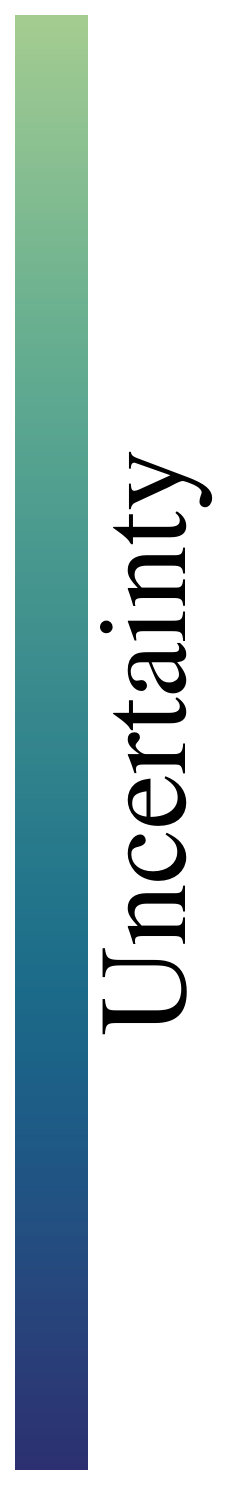}
\caption*{}
\label{fig:dragon-colorbar-right}
\end{subfigure}
\end{minipage}
\caption{A toy regression problem on the space $\c{U}_3$ of undirected graphs with $3$ nodes identified with the 3D cube graph. On these plots, color represents value of the
corresponding function and training locations are marked by dashed outline.}
\label{fig:graph_regression}
\end{figure*}

To avoid confusion, we will refer to the graph whose nodes are graphs as the \emph{metagraph} and denote it by $\c{G}$.
We propose the following natural structure for $\c{G} = (\c{V}, \c{E})$.
It is an unweighted undirected graph such that $(x, y) \in \c{E}$ if and only if the graph $y$ can be obtained from the graph $x$ by adding or deleting a single edge.
The notion of closeness defined by this structure corresponds to the Hamming distance, the number of differing bits in the representations of the graphs $x, y \in \c{V}$ as vectors in the set $\cbr{0,1}^{d}$.

This structure, however, also defines some notion of symmetries: the $\c{G}$'s own group of automorphisms $\f{Aut}(\c{G})$ that consists of all bijective maps $\phi: \c{V} \to \c{V}$ such that $(\phi(x), \phi(y)) \in \c{E} \iff (x, y) \in \c{E}$ for all  $x, y \in \c{V}$.

It is natural to ask that the distribution of a Gaussian process prior on $\c{V}$ is invariant with respect to symmetries from $\f{Aut}(\c{G})$, similar to how the standard Euclidean squared exponential and Mat\'ern priors are invariant with respect to translations and rotations, i.e., to the Euclidean isometries.

\begin{definition}
We call $f \sim \f{GP}(0, k)$ on $\c{V}$ \emph{isotropic} if and only if for all $\phi \in \f{Aut}(\c{G})$ and finite sets
$x_1, \ldots, x_n \in \c{V}$
\[
f(x_1, \ldots, x_n)
\stackrel{\text{d}}{=}
f(\phi(x_1), \ldots, \phi(x_n))
\]
where $\stackrel{\text{d}}{=}$ denotes equality in distribution.
Or, equivalently, 
\[
k(\phi(x), \phi(y)) = k(x, y)
\]
for all
$\phi \in \f{Aut}(\c{G})$
and for all pairs
$x, y \in \c{V}$.
We call the kernel of an isotropic process an \emph{isotropic kernel}.
\end{definition}

As we will show in~\Cref{sec:graph_gps:isotropic}, this class includes graph Mat\'ern Gaussian processes $\f{GP}(0, k_{\nu, \kappa, \sigma^2})$ on $\c{G}$ in the sense of \textcite{borovitskiy2021}, including, as special case for $\nu = \infty$, the Gaussian processes with diffusion (heat, squared exponential) kernels on hypercube graphs studied by \textcite{kondor2002}.
If we define $\m{\Delta}$ to be the graph Laplacian of $\c{G}$ and $(\lambda_j, f_j)$ to be its eigenpairs, where $\cbr{f_j}$ form an orthonormal basis in the set $L^2(\c{V}) = \R^{\abs{\c{V}}}$ of functions on the finite set $\c{V}$, these two kernels are given by
\[ \label{eqn:matern}
k_{\nu, \kappa, \sigma^2}(x, y)
=
\sum_{j=1}^{\abs{\c{V}}} \Phi_{\nu, \kappa, \sigma^2}(\lambda_j) f_j(x) f_j(y),
\\
\Phi_{\nu, \kappa, \sigma^2}(\lambda)
=
\begin{cases}
\frac{\sigma^2}{C_{\infty,\kappa}} e^{- \frac{\kappa^2}{2} \lambda},
&
\text{for } \nu = \infty, \\
\frac{\sigma^2}{C_{\nu,\kappa}} \del{\frac{2 \nu}{\kappa^2} + \lambda}^{-\nu}
&
\text{for } \nu < \infty.
\end{cases}
\]
In \Cref{fig:graph_regression} we illustrate a Gaussian process regression with a Matérn kernel ($\nu = 3.5$) on the set $\c{V} = \c{U}_3$ that will be made possible thanks to the developments of this section.

Clearly, any function $\Phi: \R \to [0, \infty)$ placed instead of $\Phi_{\nu, \kappa, \sigma^2}$ defines a valid covariance kernel.
An implicit degree of freedom here is the definition of $\m{\Delta}$.
If $\m{A}$ is the adjacency matrix of $\c{G}$ and $\m{D}$ is its diagonal degree matrix with $\m{D}_{i i} = \sum_{j=1}^{\abs{\c{V}}} \m{A}_{i j}$, there are three popular definitions of $\m{\Delta}$: (1) the usual Laplacian $\m{\Delta} = \m{D} - \m{A}$, (2) the random walk normalized Laplacian $\m{\Delta}_{rw} = \m{I} - \m{D}^{-1} \m{A}$ and (3) the symmetric normalized Laplacian $\m{\Delta}_{sym} = \m{I} - \m{D}^{-1/2} \m{A} \m{D}^{-1/2}$.
The class of all kernels given in form of~\Cref{eqn:matern} with any $\Phi: \R \to [0, \infty)$ we term \emph{$\Phi$-kernels}.\footnote{For our specific graph $\c{G}$ we have $\m{\Delta}_{rw} = \m{\Delta}_{sym} = \m{\Delta} / d$, hence $f_j$ are always the same and $\lambda_j$ only differ by a constant factor of~$d$, defining the same class of $\Phi$-kernels.}

Notice that although absolutely explicit, expressions like~\Cref{eqn:matern} are ostensibly intractable: they entail (1) solving the eigenproblem for the $2^d \x 2^d$-sized Laplacian matrix and (2) summing up the $2^d$ terms together.
To come with an efficient computational algorithm and to better understand the properties of these and similar kernels we proceed to study the class of isotropic processes on $\c{V}$.

\subsection{The Class of Isotropic Gaussian Processes}
\label{sec:graph_gps:isotropic}

To characterize the class of isotropic processes on $\c{V}$ we need to introduce a few additional notions.
First, we define the group $\Z_2^m$ to be the set $\cbr{0, 1}^m$ endowed with the bitwise XOR operation, which we denote by $\pd$.
Obviously, this turns the set $\c{V}$ itself into the group $\Z_2^d$.
Second, for an element $\sigma$ of the group of permutations $\Sym_d$ of a $d$-sized set define $\sigma(x) = (x_{\sigma(1)}, \ldots, x_{\sigma(d)})$ where $x$ is a graph regarded as an element of~$\cbr{0, 1}^d$.
With this, we have the following characterization of isotropic processes.

\begin{restatable}{theorem}{ThmIsotropic} \label{thm:isotropic}
A Gaussian process $f \sim \f{GP}(0, k)$ on $\c{V}$ is isotropic if and only if for all $x, y, z \in \c{V} = \Z_2^d$ and $\sigma \in \Sym_d$
\begin{align*}
(i):
&&
k(x \pd z, y \pd z) = k(x, y),
\\
(ii):
&&
k(\sigma(x), \sigma(y)) = k(x, y).
\end{align*}
\end{restatable}
\begin{proof}
This is a corollary of describing $\f{Aut}(\c{G})$ as the semidirect product~$\Z_2^d \rtimes \Sym_d$ --- see the definition, relevant discussion and the detailed proof in~\Cref{appdx:graphs}.
\end{proof}

This characterization mirrors the description of isotropic Gaussian processes on $\R^d$ as having kernels invariant to all translations, similar to (i), and to all rotations, similar to (ii).\footnote{In fact, $\Z_2^d$ and $\Sym_d$ may be regarded as the groups of translations and rotations related to vector spaces over the \emph{field~of char\-ac\-ter\-is\-tic~$2$} in place of the field $\R$, as discussed in~\Cref{appdx:graphs}.}
Moreover, this characterization may be leveraged to obtain an explicit description of the isotropic kernels class.

\begin{restatable}{theorem}{ThmIsotrKernel} \label{thm:isotr_kernel}
$f \sim \f{GP}(0, k)$ is isotropic on $\c{V}$ if and only if
\[ \label{eqn:isotr_kernel}
k(x, y)
=
\sum_{j=0}^d
\alpha_j
\!\!\!\!
\sum_{T \subseteq \cbr{1, \ldots, d}, \abs{T} = j}
\!\!\!\!\!\!\!\!
w_T(x) w_T(y),
&&
\alpha_j \geq 0,
\]
where $w_T(x) = (-1)^{\sum_{t \in T} x_t}$ are the Walsh functions --- the analogs of complex exponentials in the Fourier analysis of Boolean functions \cite{odonnell2014}.
\end{restatable}
\begin{proof}
From (i) it is possible to deduce that
\[
k(x, y) = \sum_{T \subseteq \cbr{1, \ldots, d}} \alpha_T w_T(x) w_T(y),
&&
\alpha_T \geq 0.
\]
Then, (ii) yields $\alpha_{T} = \alpha_{T'}$ for all sets $T$, $T'$ of the same size.
See the detailed proof in~\Cref{appdx:graphs}.
\end{proof}

\begin{figure*}
\centering
\includegraphics[width=0.99\textwidth]{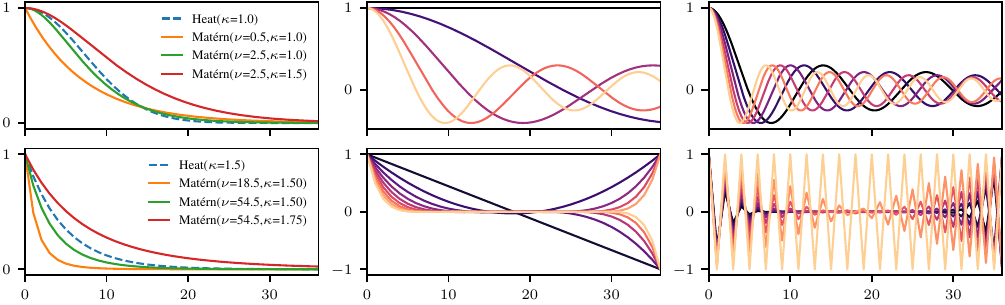}
\caption{Isotropic kernels (left) and their low (center) and high (right) frequency components, as functions of distance, in the classical Euclidean case, for $\R^2$ (top row), and in the graph space case, on $\c{DL}_6$ (bottom row).
The basis functions on $\c{DL}_6$ are $G_{36, j, m}$ as functions of $m$ for various $j$.
The basis functions on $\R^2$ arise from Hankel transform~\cite{bracewell2000} and coincide with $J_0(2 \pi q r)$ as functions of $r$ for various $q$, where $J$ denotes a Bessel function of the first kind.}
\label{fig:graph_kernels}
\end{figure*}

On the other hand, we can prove that the class of $\Phi$-kernels has exactly the same form.

\begin{restatable}{theorem}{ThmLapSpectrum} \label{thm:lap_spectrum}
The set $\cbr{w_T(\cdot)}_{T \subseteq \cbr{1, \ldots, d}}$ is an orthonormal basis of $L^2(\c{V})$ consisting of eigenfunctions of any of the Laplacians $\m{\Delta}, \m{\Delta}_{rw}, \m{\Delta}_{sym}$ on $\c{G}$ with eigenvalues
\[
\lambda_T^{\m{\Delta}} = 2 \abs{T}
&&
\lambda_T^{\m{\Delta}_{rw}} = \lambda_T^{\m{\Delta}_{sym}} = 2 \abs{T} / d
\]
\end{restatable}
\begin{proof}
The key is to recognize $\c{G}$ as the \emph{Cayley graph} of the group $\Z_2^d$.
See the detailed proof in~\Cref{appdx:graphs}.
\end{proof}

\begin{corollary*}
    The class of isotropic kernels on $\c{V}$ and the class of $\Phi$-kernels on $\c{G}$ coincide.
\end{corollary*}

\subsection{Efficient Computation of Isotropic Kernels}

\Cref{thm:lap_spectrum} relieves us of the need to solve the eigenproblem for $\m{\Delta}$ numerically.
Although this saves us an immense number of $O(2^{3 d})$ computational operations, computing $k(x, y)$ still entails summing up $2^d$ terms --- an impossible problem for even a moderate $d$.
For general graphs, \textcite{borovitskiy2021} recommend approximating $k(x, y)$ by truncating the sum in~\Cref{eqn:matern}.
Below we show that a much more efficient solution exists.

Denote the inner sum in~\Cref{eqn:isotr_kernel} by
\[
G_{d, j}(x, y)
=
\sum_{T \subseteq \cbr{1, \ldots, d}, \abs{T} = j}  w_{T}(x) w_{T}(y).
\]
For $x \in \cbr{0, 1}^d$ define $\abs{x} = \sum_{i=1}^d x_d$.
Then we have the following characterization of $G_{d, j}(x, y)$.

\begin{restatable}{theorem}{EffComp} \label{thm:eff_comp}
$G_{d, j}(x, y) = G_{d, j, m}$ where $m = \abs{x \pd y}$ is the Hamming distance between graphs $x$ and $y$.
Moreover,
\[ \label{eq:g-d-j-m-recurrence}
G_{d, j, m}
=
G_{d-1, j, m-1}-G_{d-1, j-1, m-1}
\]
with $G_{1, j, m} = (-1)^m$, $G_{d, 0, m} = 1$ and $G_{d, j, 0} = \binom{d}{j}$.
\end{restatable}
\begin{proof}
See~\Cref{appdx:graphs}.
\end{proof}

This means that $G_{d, j, m}$ may be computed by a simple dynamical programming procedure.
Moreover, caching $G_{d, j, m}$ makes kernel evaluation for arbitrary $\Phi$ take only $O(d)$ computational operations.
This can be further reduced to $O(1)$ by truncating $G_{d, j, m}$.\footnote{This truncation will implicitly incorporate an orders of magnitude larger number of Walsh function terms than the naïve one.}. This is very important for optimizing Gaussian process hyperparameters like length scale of Mat\'ern kernels, when the kernels are evaluated with different functions $\Phi$ each iteration.
We illustrate the graph and the Euclidean isotropic kernels on~\Cref{fig:graph_kernels}.

Interestingly, by virtue of~\Cref{eqn:kraw} in~\Cref{appdx:graphs}, functions $G_{d, j, m}$ coincide with the Kravchuk polynomials defined, e.g., in \textcite{macwilliams1977}. 

As a byproduct of sorts, we established that for an isotropic Gaussian process $\f{GP}(0, k)$ we have
\[
k(x, y) = \Bbbk(\abs{x \pd y})
&&
\text{for some } \Bbbk: \cbr{1, .., d} \to \R
\]
showing that isotropic kernels do indeed respect the notion of closeness given by Hamming distance.

What is more, since for $x \in \cbr{0,1}^d$ we have $\abs{x} = \norm{x}^2$, we see that for any $g: [0, \infty) \to [0, \infty)$ such that $\tilde{k}(x, y) = g(\norm{x-y}^2)$ is positive semidefinite for $x, y \in \R^d$, its restriction $k(x, y) = g(|x \pd y|)$ will be positive semidefinite and isotropic for $x, y \in \c{V} = \cbr{0, 1}^d$.

Finally, since the heat kernel on the direct product of graphs is the product of their heat kernels, by representing $\c{G}$ as the iterated product of the two-vertex complete graphs, one can infer \cite{kondor2002} that
\[ \label{eqn:heat_explicit}
k_{\infty, \kappa, \sigma^2}(x, y)
=
\sigma^2
\tanh(\kappa^2/2)^{|x \pd y|}
\]
which coincides with the Euclidean squared exponential 
\[
\sigma^2 \exp\del[2]{\!-\frac{\norm{x-y}^2}{2 \kappa'}}
~
\text{with}
~~
\kappa' \!\!=\! \frac{1}{2 \log(\tanh(\kappa^2/2))}.
\]
  
\subsection{Sampling from Isotropic Gaussian Processes}

Consider an isotropic Gaussian process $f \sim \f{GP}(0, k)$.
Since we are able to evaluate $k$ efficiently, given any collection $x_1, \ldots, x_m \in \c{V}$ denoted by $\v{x}$, sampling $f(\v{x})$ may be performed by
\[
f(\v{x}) = \m{K}_{\v{x} \v{x}}^{1/2} \v{\varepsilon}
&&
\v{\varepsilon} \sim \f{N}(0, \m{I}).
\]
Computing $\m{K}_{\v{x} \v{x}}^{1/2}$, however, requires $O(m^3)$ computational operations and is thus inefficient for larger $m$.

\begin{figure*}
\begin{subfigure}{0.30\textwidth}
\centering
\includegraphics[height=25ex]{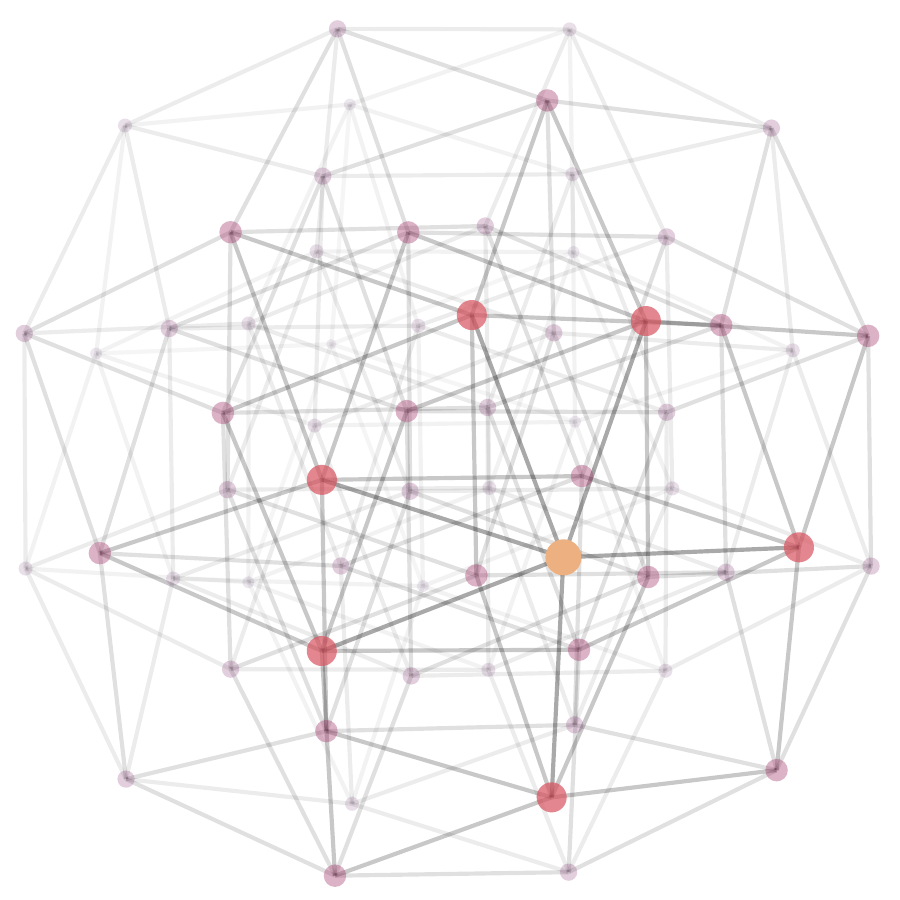}
\caption{Hypercube graph}
\end{subfigure}
\hfill
\begin{subfigure}{0.30\textwidth}
\centering
\includegraphics[height=25ex,width=25ex]{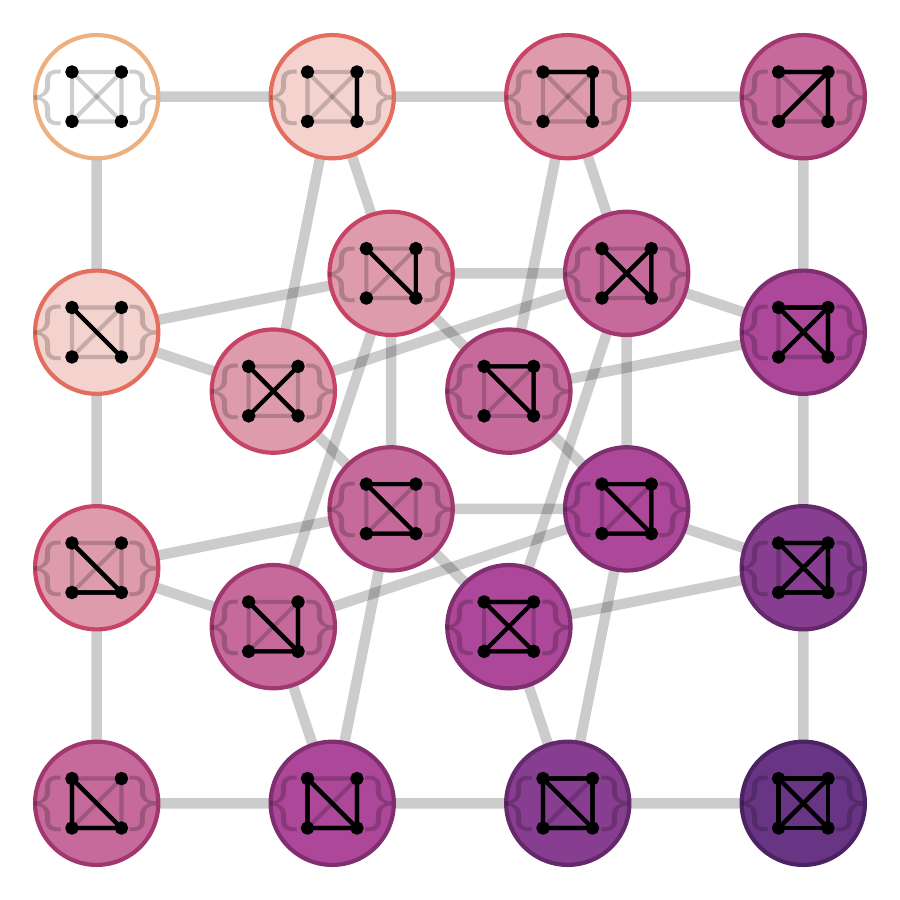}
\caption{Quotient graph under $\sim_H$} 
\end{subfigure}
\hfill
\begin{subfigure}{0.30\textwidth}
\centering
\includegraphics[height=25ex,width=25ex]{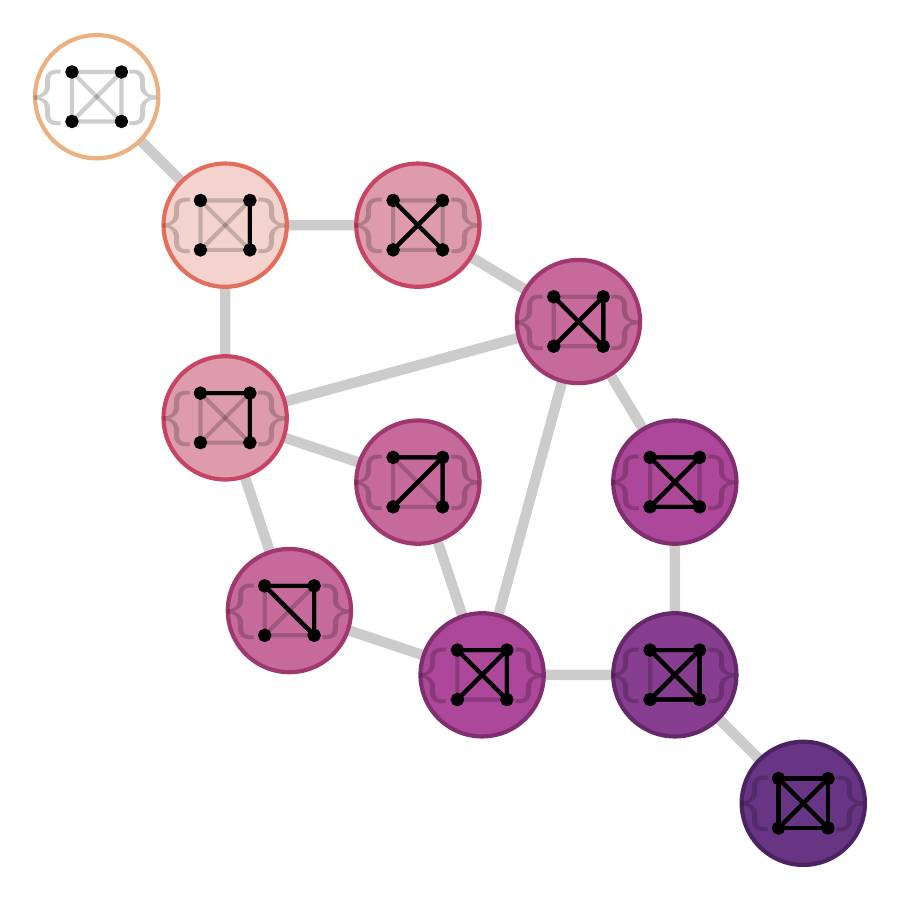}
\caption{Quotient graph under $\sim_{S_4}$, i.e. $\isom$}
\end{subfigure}
\caption{Mat\'ern kernel with $\nu = 8.5$ on the hypercube graph $\c{V} = \c{U}_4$ and on the quotient graphs corresponding to $\c{V}_{/H}$ and $\c{V}_{/ \Sym_4}$, where $H = \Sym_3 \x \Sym_1 \subseteq \Sym_4$, i.e. $x \sim_H y$ if and only if the first three vertices of $x$ may be permuted to make $x$ into $y$. Color of a node represents kernel's value between this node and the (equivalence class of) the empty graph.}
\label{fig:graphs}
\end{figure*}

Let us examine alternatives for the setting at hand.
Consider $k(x, y) = \sum_{j=0}^{d} \alpha_j G_{d, j}(x, y)$ with $\alpha_j \geq 0$ sorted in descending order.
As by~\textcite{borovitskiy2021}, one alternative is to use the approximation
\[ \label{eqn:sampling_ff}
f(x)
\!\approx\!
\sum_{j=0}^J
\sqrt{\alpha_j}
\!\!\!\!\!\!\!\!\!\!
\sum_{T \subseteq \cbr{1, \ldots, d}, \abs{T} = j}
\!\!\!\!\!\!\!\!\!\!\!\!\!\!
\varepsilon_T w_T(x),
&&
\varepsilon_T \stackrel{\text{i.i.d.}}{\sim} \f{N}(0,1).
\]
\looseness -1 This has the advantage of defining a sample at all $x \in \c{V}$ at once, not on a pre-specified small set $x_1, \ldots, x_m \in \c{V}$ as before, but with the downside of being approximate.
The number of terms in the inner sum grows roughly like $d^j$, thus~\Cref{eqn:sampling_ff} can be realistically evaluated only for small values of $J$.
If using larger values of $J$ is desirable, one approach, similar to the generalized random phase Fourier features of \textcite{azangulov2022} is
\[ \label{eqn:sampling_grpff}
f(x)
\approx
\sum_{j=0}^J
\frac{\sqrt{\alpha_j}}{\sqrt{L}}
\sum_{l=1}^L
\varepsilon_{j, l} G_{d, j}(x, u_l),
\\
\varepsilon_{j, l} \stackrel{\text{i.i.d.}}{\sim} \f{N}(0,1),
\qquad
u_l \stackrel{\text{i.i.d.}}{\sim} \f{U}(\c{V}).
\]
where $\f{U}(\c{V})$ denotes the uniform distribution over the finite set $\c{V}$.
This way one can safely take $J = d$, but one also needs to choose the value of $L$.

These sampling techniques might be useful to perform non-conjugate learning via doubly stochastic
variational inference as briefly reviewed in the relevant setting by \textcite{borovitskiy2021}, to evaluate various Bayesian optimization acquisition functions or for visualization purposes etc.

\section{Priors on Spaces of Equivalence Classes}
\label{sec:equiv_gps}

Here we define and study natural Gaussian process priors on sets of equivalence classes of graphs.
Recall that we start with some finite graph set $\c{V}$ and a subgroup $H \subseteq \Sym_n$ of the group $\Sym_n$ of permutations of vertices.
This subgroup defines an equivalence relation $\sim_H$ on $\c{V}$ and the set of equivalence classes of graphs $\c{V}_{/ \sim_H}$.
For example, if $H = \Sym_n$, then $\sim_H$ is the graph isomorphism relation.
For convenience, we will write $\c{V}_{/ H}$ as a shorthand for~$\c{V}_{/ \sim_H}$.
We refer the reader to~\Cref{sec:intro:input_spaces} for a more detailed discussion on equivalence classes.

Our strategy to build a prior on $\c{V}_{/ H}$ is to take a prior on $\c{V}$ that is provided to us by~\Cref{sec:graph_gps}, and \emph{project} it, making it constant on each of the equivalence classes, i.e. such that $f(x) = f(y)$ for all $x \sim_H y$.

Consider the space $L^2(\c{V})$ of functions on $\c{V}$ and the operator $\f{Pr}: L^2(\c{V}) \to L^2(\c{V})$ given by
\[
(\f{Pr} f)(x)
=
\frac{1}{\abs{H}}
\sum_{\sigma \in H}
f(\sigma(x))
\]
Denote by $W$ the subspace of $L^2(\c{V})$ consisting of functions that are constant on each of the equivalence classes $\overline{x}$, $x \in \c{V}$.
We prove in~\Cref{appdx:classes} that the operator $\f{Pr}$ is an orthogonal projector onto the $\abs[0]{\c{V}_{/ H}}$-dimensional space $W$.

Now we define the prior on $\c{V}_{/ H}$ corresponding to a prior on~$\c{V}$, which we call its \emph{$H$-invariant version}.

\begin{definition}
Consider a Gaussian process $f \sim \f{GP}(0, k)$ on~$\c{V}$.
We call the Gaussian process
\[ \label{eqn:hinv_def}
f_{/H}(x) = (\Pr f)(x)
\quad
\text{with kernel}
\\
k_{/H}(x, y) = \frac{1}{\abs{H}^2} \sum_{\sigma_1 \in H} \sum_{\sigma_2 \in H} k(\sigma_1(x), \sigma_2(y))
\]
the \emph{$H$-invariant} version of $f$ and the kernel $k_{/H}$ the \emph{$H$-invariant} version of $k$.
\end{definition}

The term is justified because, as it is easy to see, for all $\sigma, \sigma_1, \sigma_2 \in H$ and all $x, y \in \c{V}$ we have
\[
f_{/H}(\sigma(x)) &= f_{/H}(x),
\\
k_{/H}(\sigma_1(x), \sigma_2(y)) &= k_{/H}(x, y),
\]
meaning that $f_{/H}$ and $ k_{/H}$ are constant when restricted onto each (pair of) equivalence classes.

\looseness -1 Interestingly, if we regard $H$-invariant versions as functions of equivalence classes (i.e., on the set~$\c{V}_{/ H}$) rather than functions of graphs (i.e., on the set~$\c{V}$), they may also be seen as arising from certain metagraphs, more concretely from  appropriately weighted \emph{quotient graphs} \cite{cozzo2018}.
Specifically, let us define the graph $\c{G}_{/H} = (\c{V}_{/H}, \c{E}_{/H})$ such that the weight of an edge $(\overline{x}, \overline{y})$ is exactly the cardinality of the set $\cbr{(x', y') \in \c{E} : x' \in \overline{x}, y' \in \overline{y}}$.
Then, we have following.

\begin{restatable}{theorem}{ThmQuotient} \label{thm:quotient}
Consider the class of $\Phi$-kernels induced by the symmetric normalized Laplacian,
and take a Gaussian process $f \sim \f{GP}(0, k)$ on $\c{V}$ with a $\Phi$-kernel given by a function~$\Phi$.
Its $H$-invariant version $f_{/H} \sim \f{GP}(0, k_{/H})$ has kernel $k_{/H}(x, y) = \psi(\overline{x}) \psi(\overline{y}) k_{\Phi}(\overline{x}, \overline{y})$ where $k_{\Phi}$ is the $\Phi$-kernel on the quotient graph $\c{G}_{/H}$ with the same function~$\Phi$.
Moreover, $\psi(\overline{x}) = \abs{\overline{x}}^{-1/2}$.
\end{restatable}
\begin{proof}
See~\Cref{appdx:classes}.
\end{proof}

When reinterpreting a function from $W \subseteq L^2(\c{V})$ as a function in $L^2(\c{V}_{/ H})$, the norm is not preserved.
This explains why the factor $\psi(\cdot)$ appears in the theorem above.

We illustrate Mat\'ern kernels on the graph $\c{G}$ for $\c{V} = \c{U}_4$, the quotient graph corresponding to $H = \Sym_3 \x \Sym_1 \subseteq \Sym_4$ and the quotient graph corresponding to the graph isomorphism relation $\isom$ in~\Cref{fig:graphs}.

After we introduced a way to build priors on $\c{V}_{/ H}$, we turn to the associated computational routines.

\subsection{Kernel Computation and Sampling}

We start by proving a hardness result suggesting it is impossible to compute $k_{/ H}$ exactly in an efficient way.

\begin{restatable}{theorem}{ThmHardness} \label{thm:hardness}
Assume $k$ is a $\Phi$-kernel on $\c{G}$ for a strictly positive $\Phi$.
Exactly computing $k_{/ H}$ at three pairs of inputs $(x, y), (x, x)$ and $(y, y)$ is at least as hard as checking whether $x \sim_H y$.
In particular, if $H = \Sym_n$, it is at least as hard as resolving the graph isomorphism problem.\footnote{Whether or not this may be done in polynomial time is currently considered an open problem.}
\end{restatable}

\begin{proof}
We prove that $x \sim_H y$ is equivalent to the condition
\[
k(x, y) = \frac{k(x, x) + k(y, y)}{2}.
\]
See~\Cref{appdx:classes} for details.
\end{proof}

This effectively forbids exact evaluation of $k_{/ H}$ for larger values of $n$.
For small $n$, e.g., $n \leq 10$, this can be done simply by definition.
For moderately larger values of $n$, one may consider using the Monte Carlo approximation
\[ \label{eqn:montecarlo}
k_{/H}(x, y) \approx \frac{1}{\abs{S}^2} \sum_{\sigma_1 \in S} \sum_{\sigma_2 \in S} k(\sigma_1(x), \sigma_2(y)),
\\
S \subseteq H,
\qquad
S \ni \sigma \stackrel{\text{i.i.d.}}{\sim} \f{U}(H),
\]
where $\f{U}(H)$ denotes uniform distribution over the set~$H$.
\begin{remark}
The right-hand side of~\Cref{eqn:montecarlo} is necessarily positive semidefinite.
One can use the isotropy of $k$ to introduce a different approximation, $k_{/H}(x, y) = \frac{1}{\abs{S}} \sum_{\sigma \in S} k(\sigma(x), y)$ and hope for better convergence as
\[
k_{/ H}(x, y)
&=
\frac{1}{\abs{H}^2}
\sum_{\sigma_1, \sigma_2 \in H}
k(\sigma_2^{-1}\sigma_1(x), y)
\\
&=
\frac{1}{\abs{H}}
\sum_{\sigma \in H}
k(\sigma(x), y).
\]
However, this approximation can easily fail to be positive semidefinite (even symmetric!), hence the approximation in~\Cref{eqn:montecarlo} will usually be preferable in practice.
\end{remark}

\begin{figure*}[t]
\centering
\includegraphics[width=0.95\textwidth]{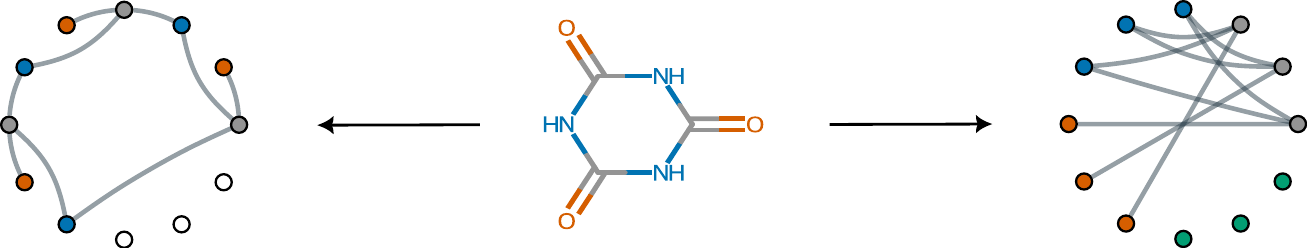}
\caption{Transforming molecules into graphs. Given a molecule (middle), we construct a graph by assigning certain atoms to arbitrary nodes (left), and by assigning atoms to certain pre-specified groups of nodes (right). Nodes (and node groups) carry the same color as the corresponding atoms.}
\label{fig:experiments}
\end{figure*}

Finally, we discuss sampling $f_{/H} \sim \f{GP}(0, k_{/ H})$.
Mirroring the approximation in~\Cref{eqn:montecarlo}, to draw a sample $\hat{f}_{/H}$ we suggest drawing a sample $\hat{f} \sim \f{GP}(0, k)$ and putting
\[
\hat{f}_{/H}(x)
=
\sum_{\sigma \in S}
\frac{1}{\abs{S}}
\hat{f}(\sigma(x)).
\]

It is trivial to check that if $\hat{f}$ is an exact sample from $\f{GP}(0, k_{/ H})$, then the covariance of the right-hand side will exactly coincide with the right-hand side of~\Cref{eqn:montecarlo}.

\vspace{-3pt}
\section{Experimental Setup}

Inspired by applications in chemistry, we evaluate our proposed models on a real molecular property prediction task.

To match typical application settings for Gaussian processes, we consider a small dataset.
Specifically, we utilize the FreeSolv \cite{mobley2014} dataset provided as part of the MoleculeNet benchmark \cite{wu2018}.
It consists of 642 molecules with experimentally measured hydration free energy values.
After removing invalid molecules, we utilize a 80/20 train/test split to obtain 510 and 128 examples for training and testing, respectively.

We also consider a subset of FreeSolv (FreeSolv-S), allowing upto four atom types (carbon, nitrogen, oxygen and chlorine), with a maximum of 3 atoms per atom type.
This results in 52 training and 13 text examples, and allows us to test exact projected Gaussian processes and demonstrate reasonable performance with even more limited data.

To construct graphs from molecules, we assign atoms to nodes, and corresponding chemical bonds to edges.
We adopt two strategies - a) assigning atoms to arbitrary nodes in the graph (\textsc{Graph-B} in \Cref{tab:molecule_experiment}, left on~\Cref{fig:experiments}), and b) assigning certain types of atoms to certain pre-specified groups of vertices thus aligning different data samples better (\textsc{Graph-A} in \Cref{tab:molecule_experiment}, A stands for aligned, right on~\Cref{fig:experiments}).
Both strategies are illustrated in~\Cref{fig:experiments}.

\begin{table}[bh!]
\caption{Molecule property prediction performance. RMSE is reported on the original data scale, with standard deviation of 3.89 on FreeSolv and 3.83 on FreeSolv-S.}
\label{tab:molecule_experiment}
\begin{center}
\resizebox{1.05\linewidth}{!}{\begin{tabular}{lcccc}
\toprule
 \multirow{2}{*}{\textbf{Method}} & \multicolumn{2}{c}{\textbf{FreeSolv}}&\multicolumn{2}{c}{\textbf{FreeSolv-S}}\\
 \cmidrule{2-3}\cmidrule{4-5}
 & Log Lik. & RMSE & Log Lik. & RMSE\\
 \midrule
 \textsc{Naive} & --- & $0.32$ & --- & $0.99 \pm 0.25$ \\
 \textsc{Linear} & $-154.37$ & $0.27$ & $-19.44 \pm 5.19$ & $1.06 \pm 0.26$ \\
 \midrule
\textsc{Graph-B} \\
 - Heat & $-151.28$ & $0.26$ & $-19.42 \pm 5.41 $ & $1.05 \pm 0.26$ \\
 - Mat\'ern & $-151.19$ & $0.26$ & $-19.65 \pm 5.11$ & $1.05 \pm 0.25$ \\
 \textsc{Graph-A} \\
 - Heat & $-77.55$ & $0.17$ & $-10.65 \pm 4.59$ & $0.67 \pm 0.31$ \\
 - Mat\'ern & $-77.54$ & $0.17$ & $-10.69 \pm 4.45$ & $0.67 \pm 0.30$ \\
 \textsc{Projected} \\
 - Heat & --- & --- & $-3.88 \pm 3.40$ & $0.50 \pm 0.20$ \\
 - Mat\'ern & --- & --- & $-4.26 \pm 3.37$ & $0.49 \pm 0.20$ \\
\midrule
\end{tabular}}
\end{center}
\end{table}

\textbf{Varying number of nodes.}
To handle graphs with varying number of nodes, we use Gaussian processes on the metagraph corresponding to the maximal number of nodes throughout the dataset.
All graphs are than "padded", if necessary, by additional disconnected nodes.
These can be seen at the bottom right of the graphs presented on~\Cref{fig:experiments}.

We compare our proposed methods to the \textsc{Naive} baseline that outputs the mean over the train set and a \textsc{Linear} kernel based Gaussian process. On the FreeSolv-S dataset, we also evaluate the projected kernels (\textsc{Projected}), where the group of $H$ is such that $x \sim_H y$ if and only if $x$ can be made into $y$ by permuting only atoms of the same type (carbons with carbon, oxygens with oxygen etc.). Results are displayed in~\Cref{tab:molecule_experiment}, with more details in~\Cref{appdx:experiments}, and discussed, along with the theoretical contributions, in~\Cref{sec:discussion_conclusion}. For the Freesolv-S dataset, we generate 10 random splits and report the mean and standard deviation of the metrics across these splits.

\vspace*{-\baselineskip}

\section{Discussion and Conclusion}
\label{sec:discussion_conclusion}
The ease of use of the graph heat kernel, granted by a closed form formula in~\Cref{eqn:heat_explicit}, and its fair performance observed in experiments make it a good baseline kernel to consider in the absence of other structure.
The theory of this paper places it onto a firmer foundation, showing how it arises from endowing graph sets with a natural geometry.

Graph Matérn kernels and the heat kernel may slightly outperform each other depending on the setting and the metric.
It is therefore interesting to characterize settings where and to what extent one kernel is better than the other.
Explicit spectral representations we derived both for finite-$\nu$ Matérn and heat kernels may be key to this, as they often appear in providing regression convergence rates \cite{kanagawa2018}.
These spectral representations may also be relevant for inferring regret bounds in Bayesian optimization \cite{srinivas2010} or related active learning techniques.

Our results suggest that projected graph Gaussian processes may heavily outperform simple graph Gaussian processes in problems possessing inherent invariances, where performance of the simple models is hindered by varied graph alignments --- though appropriate data preprocessing can already make a difference.
Since exact computation of projected kernels is precluded by \Cref{thm:hardness}, further research into building efficient approximations thereof --- which are not forbidden by the result --- is needed.

Finally, we believe that the geometric framework treating graph spaces as vertex sets of appropriate metagraphs is quite general and a promising direction for further development.
For example, we anticipate it to easily extend to certain settings with discretely labeled edges, where the metagraph can be chosen in such a way that isotropic processes are described in terms of other simple bases, similar to how they are described by the Walsh system in this paper.

To conclude, we hope that the proposed framework and our current developments set up the scene and guide further research interest for Gaussian process based modeling on spaces of graphs and their corresponding prospective applications like Bayesian optimization, active learning, etc.

\section*{Acknowledgments}
\looseness -1 The authors are grateful to Prof. Risi Kondor (The University of Chicago) and Dr. Konstantin Golubev (Google) for fruitful discussions.
This work was supported by the European Research Council (ERC) under the European Union’s Horizon 2020 research and innovation programme grant aggreement No 815943, the NCCR Catalysis (grant number 180544), a National Centres of Competence in Research funded by the Swiss National Science Foundation, and an ETH Z\"urich Postdoctoral Fellowship to VB. VRS acknowledges funding from the European Union’s Horizon 2020 research and innovation programme 826121.

\printbibliography

\newpage
\onecolumn
\appendix

\input{appendix}

\end{document}

%% file: appendix.tex

\section{Gaussian Processes on Finite Spaces of Graphs}
\label{appdx:graphs}

As in~\Cref{sec:graph_gps}, define $\c{V}$ to be one of the sets $\c{U}_n, \c{UL}_n, \c{D}_n$, or $\c{DL}_n$ of unweighted graphs on $n$ nodes, identified with $\cbr{0,1}^d$ for the appropriate $d$.
The metagraph $\c{G}$ from~\Cref{sec:graph_gps} may be recognized to be the \emph{hypercube graph} \cite{kondor2002}, also referred to as the \emph{$d$-cube graph} or just a \emph{cube graph} \cite{ovchinnikov2011}.
We start by describing the group $\f{Aut}(\c{G})$ of its automorphisms.
To do this, we need to introduce several notions from the group theory.

First, recall that a subgroup $N$ of a group $G$ is called \emph{normal} if and only if $g x g^{-1} \in N$ for all $g \in G$ and $x \in N$.
Now we formally introduce the semidirect product of groups \cite{robinson2003}.

\begin{definition}
Given a group $G$, a subgroup $H \subseteq G$ and a normal subgroup $N \subseteq G$,
we say that $G$ is the \emph{semidirect product of $N$ and $H$} if and only if $G = N H$ and $N \cap H = e$ where $e \in G$ is the identity.
In this case we write $G = N \rtimes H$.
\end{definition}
Note that for all $g \in G = N \rtimes H$ there are unique $n \in N$ and $h \in H$ such that $g = n h$ \cite{robinson2003}.
With this we can characterize the group $\f{Aut}(\c{G})$ of automorphisms of the metagraph $\c{G}$.

\begin{result} \label{thm:aut_group}
$\f{Aut}(\c{G}) = \Z_2^d \rtimes \Sym_d$ where the groups $\Z_2^d$ and $\Sym_d$ were introduced~\Cref{sec:graph_gps}.
This semidirect product has its own name, called the \emph{hyperoctahedral group}.\footnote{The hyperoctahedral group is defined as the \emph{wreath product} $\Sym_2 \wr \Sym_d = \Sym_2^d \rtimes \Sym_d = \Z_2^d \rtimes \Sym_d$ \cite{kerber2006}.}
\end{result}
\begin{proof}
See~\textcite[Theorem 3.31]{ovchinnikov2011}.
\end{proof}

This mirrors the Euclidean space, where the group of isometries of $\R^d$ is ${\f{E}(d) = \R^d \!\rtimes\! \f{O}(d)}$ where $\R^d$ is the Euclidean addition group acting on $\R^d$ by translations and $\f{O}(d)$ is the group of orthogonal matrices acting on $\R^d$ by rotations.
There is an even deeper connection here.
The set $\cbr{0, 1}$ endowed with addition and multiplication modulo $2$ becomes the \emph{field of characteristic $2$} \cite{dummit2004}.
The addition group of the corresponding vector space $\cbr{0, 1}^d$ coincides with the group $\Z_2^d$, while all the orthogonal matrices over this field consist of only $0$ or $1$ entries, meaning that they are permutation matrices, turning the group $\f{O}(d)$ into $\Sym_d$ in this case.

We are now ready to prove~\Cref{thm:isotropic} from~\Cref{sec:graph_gps}.

\ThmIsotropic*

\begin{proof}
This is an almost immediate corollary of~\Cref{thm:aut_group}.
Since all translations $\cdot \to \cdot \pd z$ and all permutations $\cdot \to \sigma(\cdot)$ are automorphisms of the metagraph~$\c{G}$, the forward implication is obvious.

The backward implication follows from the fact that each element $g \in \f{Aut}(\c{G})$ can be uniquely represented as a product $t h$ where $t$ is a translation by some $z \in \Z_2^d$ and $h$ is a permutation $\sigma \in \Sym_n$, this is a basic property of the semidirect product $\Z_2^d \rtimes \Sym_d$.
\end{proof}

We will further study the metagraph $\c{G}$ through the following notion.

\begin{definition}
Let $G$ be a group, and let $S \subseteq G$ be such subset that $S = S^{-1}$ and $e \not\in S$, where $e \in G$ is the identity element.
Then the \emph{Cayley graph} $\cayley(G, S)$ is the graph with vertex set $V = G$ and with the edge set $E$ such that $(g, g') \in E$ if and only if $g' g^{-1} \in S$ \cite{godsil2001}.
\end{definition}

Define $\v{e}_i \in \Z_2^d$ to be the vectors of zeroes and ones where there is a single $1$, standing at the $i$th position.
Denote $S = \cbr{\v{e}_1, \ldots, \v{e}_d}$.
Our metagraph $\c{G}$, the hypercube graph, may be recognized to be the Cayley graph $\cayley(\Z_2^d, S)$ of the group $\Z_2^d$ with this specific $S$ \cite{godsil2001}.

Consider the space $L^2(\Z_2^d)$ of real-valued functions $f: \Z_2^d \to \R$ with inner product given by
\[
\innerprod{f}{g}_{L^2(\Z_2^d)}
=
\frac{1}{2^d} \sum_{x \in \Z_2^d} f(x) f(x)
\]
This space admits \cite{odonnell2014} an orthonormal basis of \emph{characters} of the group $\Z_2^d$, i.e., functions $\chi: \Z_2^d \to \cbr{0, 1}$ with ${\chi(x \pd y) = \chi(x) \chi(y)}$.
Moreover, this basis is known to be the family $\cbr{w_T}_{T \subseteq \cbr{1, \ldots, d}}$ of Walsh functions given by
\[
w_T(x) = (-1)^{\sum_{t \in T} x_t} \in \cbr{-1, 1}
\]
Note that the number of such functions is $2^{d}$ that matches the dimension of $L^2(\Z_2^d)$.
Representing functions from $L^2(\Z_2^d)$ in terms of this basis is similar to representing periodic functions in form of the classical Fourier series.

Consider the adjacency matrix $\m{A}$ of $\c{G} = \cayley(\Z_2^d, S)$.
We may interpret it as the operator
\[
\m{A}: L^2(\Z_2^d) \to L^2(\Z_2^d)
&&
\text{given by}
&&
(\m{A} f)(x) = \sum_{y \in \Z_2^d} \m{A}_{x y} f(y).
\]
Now we prove the following auxiliary result.

\begin{lemma}
The Walsh functions $w_T: \Z_2^d \to \R$ are eigenfunctions of the operator $\m{A}$ corresponding to eigenvalues $\lambda_T = d - 2 \abs{T}$.
That is $\m{A} w_T = \lambda_T w_T$ with $\lambda_T = d - 2 \abs{T}$.
\end{lemma}

\begin{proof}
For $x \in \Z_2^d$ its inverse $\md x$ is equal to $x$ itself, hence the condition $g' g^{-1} \in S$ from the definition of a Cayley graph turns into $y \md x = y \pd x = \v{e}_i$ for some $i \in \cbr{1, \ldots, d}$.
We have
\[
(\m{A} w_T)(x)
=
\sum_{y \in \Z_2^d} \m{A}_{x y} w_T(y)
=
\sum_{i=1}^d w_T(x \pd \v{e}_i)
=
w_T(x) \sum_{i=1}^d w_T(\v{e}_i).
\]
Write
\[
w_T(\v{e}_i)
=
(-1)^{\sum_{t \in T} (\v{e}_i)_t}
=
\begin{cases}
-1 & i \in T \\
1  & \text{otherwise.}
\end{cases}
\]
Hence, $\lambda_T = \sum_{i=1}^d w_T(\v{e}_i) = (d - \abs{T}) - \abs{T} = d - 2 \abs{T}$ which proves the claim.
\end{proof}

Now we are ready to prove~\Cref{thm:lap_spectrum}

\ThmLapSpectrum*
\begin{proof}
Recall that the Laplacian matrix $\m{\Delta}$ is given by $\m{\Delta} = \m{D} - \m{A}$ where $\m{D}$ is the diagonal degree matrix.
For the metagraph $\c{G}$ we have $\m{D} = d \m{I}$.
It follows that $\m{\Delta}$, interpreted as the operator of the form $\m{\Delta}: L^2(\Z_2^d) \to
L^2(\Z_2^d)$ has $\cbr{w_T}_{T \subseteq \cbr{1, \ldots, d}}$ as its basis of eigenfunctions and
\[
\lambda_j = d - d + 2 \abs{T} = 2 \abs{T}
\]
as the corresponding eigenvalues. 

Recall that $\m{\Delta}_{rw} = \m{I} - \m{D}^{-1} \m{A}$ and $\m{\Delta}_{sym} = \m{I} - \m{D}^{-1/2} \m{A} \m{D}^{-1/2}$ which in our case means that $\m{\Delta}_{rw} = \m{\Delta}_{sym} = \m{\Delta}/d$, hence the remaining part of the claim.
\end{proof}

With these considerations, we are ready to prove the remaining theorems of~\Cref{sec:graph_gps}.

\ThmIsotrKernel*
\begin{proof}
Since $\cbr{w_T}_{T \subseteq \cbr{1, \ldots, d}}$ is an orthonormal basis of $L^2(\Z_2^d)$ we can write
\[
f(x)
=
\sum_{T \subseteq \cbr{1, \ldots, d}}
\varepsilon_T w_T(x),
&&
\varepsilon_T
=
\innerprod{f}{w_T}
=
\frac{1}{2^d}
\sum_{y \in Z_2^d}
f(y) w_T(y),
\]
where $\varepsilon_T$ are Gaussian random variables.
Of course $\E \varepsilon_T = 0$.
Write
\[
\Cov(\varepsilon_T, \varepsilon_{T'})
&=
\frac{1}{2^{2d}}
\sum_{x, y \in Z_2^d}
\Cov(f(x), f(y)) w_{T}(x) w_{T'}(y)
\\
&=
\frac{1}{2^{2d}}
\sum_{x, y \in Z_2^d}
k(x, y) w_T(x) w_{T'}(y)
\]
Note that for any $g \in L^2(Z_2^d)$ and any $z \in \Z_2^d$ we have $\sum_{y \in \Z_2^d} g(z \pd y) = \sum_{y \in \Z_2^d} g(y)$.
Denoting $\v{0}$ to be the graph with no edges, we have
\[
\Cov(\varepsilon_T, \varepsilon_{T'}) &=
\frac{1}{2^{2d}}
\sum_{x \in Z_2^d}
\sum_{y \in Z_2^d}
k(x, x \pd y) w_T(x) w_{T'}(x \pd y)
\\
&=
\frac{1}{2^{2d}}
\sum_{x \in Z_2^d}
\sum_{y \in Z_2^d}
k(\v{0}, y) w_T(x) w_{T'}(x) w_{T'}(y)
\\
&=
\del[2]{
\frac{1}{2^{d}}
\sum_{x \in Z_2^d}
w_T(x) w_{T'}(x)
}
\del[2]{
\frac{1}{2^{d}}
\sum_{y \in Z_2^d}
k(\v{0}, y)  w_{T'}(y)
}.
\]
Since $w_T$ are orthonormal, the last equation shows that $\Cov(\varepsilon_T, \varepsilon_{T'}) = 0$ for $T \not= T'$.
Hence all $\varepsilon_T$ are independent.
If we denote their respective variances $\Var(\varepsilon_T)$ by $\alpha_T$, we get
\[
k(x, y)
=
\sum_{T \subseteq \cbr{1, \ldots, d}} \del{\Var(\varepsilon_T)} w_T(x) w_T(y)
=
\sum_{T \subseteq \cbr{1, \ldots, d}} \alpha_T w_T(x) w_T(y).
\]

Now recall that $w_T(x) = (-1)^{\sum_{t \in T} x_t}$.
For $\sigma \in \Sym_d$ write
\[
w_T(\sigma(x))
=
(-1)^{\sum_{t \in T} x_{\sigma(t)}}
=
(-1)^{\sum_{t \in \sigma(T)} x_{t}}
=
w_{\sigma(T)}(x)
\]
where $\sigma(T)$ denotes the action of $\sigma$ on the set $T$ elementwise.
It follows that
\[
k(\sigma(x), \sigma(y))
=
\sum_{T \subseteq \cbr{1, .., d}} \alpha_T w_{\sigma(T)}(x) w_{\sigma(T)}(y)
\]

Since $\cbr{w_T}_{T \subseteq \cbr{1, \ldots, d}}$ is an orthonormal basis and since $k(x, y) = k(\sigma(x), \sigma(y))$ by assumption, we have $\alpha_T = \alpha_{\sigma(T)}$ for all $\sigma \in \Sym_d$.
Hence, $\alpha_T = \alpha_{T'}$ for all $T, T' \subseteq \cbr{1, \ldots, d}$ such that $\abs{T} = \abs{T'}$.
This proves the claim.
\end{proof}

Recall that $\abs{x} = \sum_{j=1}^d x_d$ and
\[
G_{d, j}(x, y)
=
\sum_{T \subseteq \cbr{1, \ldots, d}, \abs{T} = j}  w_{T}(x) w_{T}(y).
\]
We now prove the recurrence relation governing the values of $G_{d,j}$, enabling us to use a
dynamical program to find them.

\EffComp*
\begin{proof}
First of all, by direct computatiom, $G_{d, 0, m} = 1$, $G_{1, j, m} = (-1)^m$.

Denote $z = x \pd y$ and write
\[
G_{d, j}(x, y)
=
\sum_{T \subseteq \cbr{1, \ldots, d}, |T| = j}
\!\!\!\!\!\!\!\!\!\!\! w_T(x) w_T(y)
=
\sum_{T \subseteq \cbr{1, \ldots, d}, |T| = j}
\!\!\!\!\!\!\!\!\!\!\! w_T(x \pd y)
=
\sum_{T \subseteq \cbr{1, \ldots, d}, |T| = j}
\!\!\!\!\!\!\!\!\!\!\! w_T(z)
.
\]
Suppose $z$ has $m$ ones and $d-m$ zeros.
Since the kernel is invariant with respect to all edge permutations, we can assume that $z = (1, \ldots, 1, 0, \ldots, 0)$, i.e., $z$ has $m$ ones followed by $d-m$ zeros.

The terms in the sum above are $(-1)^\ell$ if $\abs{T \cap \cbr{1, \ldots, m}} = \ell$.
Hence,
\[ \label{eqn:kraw}
G_{d, j}(x, y)
=
G_{d, j, m}
=
\sum_{\ell = \max(0, m+j-d)}^{\min(j, m)} (-1)^\ell \binom{m}{\ell} \binom{d-m}{j-\ell}.
\]
Specifically, if $m=0$, then $G_{d, j, 0} = \binom{d}{j}$.

One can also construct a recurrence relation as follows.
Recall that we assume that $z$ has $m > 0$ ones followed by $d-m$ zeros.
Dividing subsets $T$ of size $j$ into two classes, those that include index $1$ and those that do not, we write
\begin{align*}
G_{d,j}(x, y)
&=
\sum_{T: |T| = j, 1 \in T} w_T(z) + \sum_{T: |T| = j, 1 \not\in T} w_T(z) \\
& = (-1) G_{d-1, j-1, m-1} + G_{d-1, j, m-1}.
\end{align*}
Thus, $G_{d,j,m} = G_{d-1, j, m-1} - G_{d-1, j-1, m-1}$.
\end{proof}
Note that $G_{d, j, m}$ may become quite large, this is already apparent from $G_{d, j, 0} = \binom{d}{j}$.
Because of this, when implementing the dynamical program, it makes sense to consider $G_{d, j, m}' = G_{d, j, m} / \binom{d}{j}$ for which we have $G_{d, 0, m}' = 1$, $G_{1, j, m}' = (-1)^m$ same as before, but now $G_{d, j, 0}' = 1$ and
\[ \label{eq:normalized-dp-recurrence}
G_{d, j, m}'
&=
G_{d, j, m} / \binom{d}{j}
=
G_{d-1, j, m-1} / \binom{d}{j} - G_{d-1, j-1, m-1} / \binom{d}{j}
\\
&=
\frac{d-j}{d} G_{d-1, j, m-1} / \binom{d-1}{j} - \frac{j}{d} G_{d-1, j-1, m-1} / \binom{d-1}{j-1}
\\
&=
\frac{d-j}{d} G_{d-1, j, m-1}' - \frac{j}{d} G_{d-1, j-1, m-1}'.
\]

\section{Gaussian Processes on Spaces of Graph Equivalence Classes}
\label{appdx:classes}

Recall that $\c{V}$ denotes one of the sets $\c{U}_n, \c{UL}_n, \c{D}_n, \c{DL}_n$ of unweighted graphs on $n$ nodes, identified with $\cbr{0,1}^d$ for the appropriate $d$.
Recall that $H \subseteq \Sym_n$ denotes a subgroup of the node permutation group $\Sym_n$.
This $H$ induces the equivalence relation $\sim$ and the set of $\sim$-equivalence classes $\c{V}_{/ H}$.
Recall that $\overline{x} \in \c{V}_{/ H}$ denotes the equivalence class of the element $x \in \c{V}$.

Virtually all the properties of $H$-invariant versions of Gaussian processes and their kernels will be the consequences of the fact that the partition $\c{V} = \cup_{\overline{x} \in \c{V}_{/ H}} \overline{x}$ is \emph{equitable}.
We prove this, after formally introducing this notion and presenting its relevant properties, in the following subsection.

\subsection{Equitable Partitions}

Consider an unweighted undirected graph $G = (V, E)$ with adjacency matrix~$\m{A}_G$.
\begin{definition}
For a set $C \subseteq V$ and a vertex $v \in V$ define $\deg(v, C) = \abs{\cbr{v' \in C : (v, v') \in E}}$.
A~partition $V = \cup_{i=1}^m V_i$  is called \emph{equitable} if and only if
\[
\deg(v, V_i) = \deg(v', V_i)
&&
\text{for all } v, v' \in V_j
&&
\text{and all } j \in \cbr{1, \ldots, m}.
\]
\end{definition}

We may consider the adjacency matrix of $G$ as an operator $\m{A}_G: L^2(V) \to L^2(V)$.

\begin{theorem} \label{thm:equitable_basis}
Consider an equitable partition $V = \cup_{i=1}^m V_i$.
There is an orthonormal basis $\cbr{f_j}_{j=1}^{\abs{V}}$ of $L^2(V)$ consisting of eigenfunctions of $\m{A}_G$ split in two groups: $\cbr{1, \ldots, \abs{V}} = \Lambda_1 \cup \Lambda_2$ such that
\[ \label{eqn:two_groups}
\restr{f_j}{V_i} \equiv c_{i j} \quad \text{if} \quad
j \in \Lambda_1,
&&
\sum_{v \in V_i} f_j(v) = 0 \quad \text{if} \quad
j \in \Lambda_2,
&&
1 \leq i \leq m.
\]
Moreover, $\abs{\Lambda_1} = m$ and $\abs{\Lambda_2} = \abs{V} - m$.
In words, there are $m$ functions in $\Lambda_1$, all of which are piecewise-constant on the
partition, and any function in $\Lambda_2$ has zero average in each part of the partition.
\end{theorem}
\begin{proof}
This is a widely known fact (see e.g., \textcite{cozzo2018}), we present here its simple proof, as it is instructive.

Define $\1_{V_i}: V \to \cbr{0, 1}$ to be the indicator of the set $V_i$, i.e., $\1_{V_i}(v) = 1$ if $v \in V_i$ and $\1_{V_i}(v) = 0$ otherwise.
Let us denote by $W = \Span\cbr{\1_{V_i}}_{i=1}^m \subseteq L^2(V)$ the space of functions constant on all sets $V_j$.
It is easy to check that
\[
(\m{A}_G \1_{V_i})(v)
=
\deg(v, V_i).
\]
Since the partition is equitable, the right-hand side, as a function of $v$, is constant on all sets $V_j$ and thus $\m{A}_G \1_{V_i} \in W$.
This means that $W$ is an invariant subspace of the operator $\m{A}_G$.

Since $\m{A}_G$, as a matrix, is symmetric, there exists an orthonormal basis of $L^2(V)$ consisting of its eigenvectors (eigenfunctions).
From classical linear algebra we know that it may be chosen to consist of vectors split in two groups: vectors belonging to the invariant space $W$ and vectors belonging to its orthogonal complement $W^{\orth}$.
The former are constant on all sets $V_j$, meaning that for them the left part of~\Cref{eqn:two_groups} holds, while the latter, since orthogonal to $W$, average to zero over all sets $V_j$, i.e., the right part of~\Cref{eqn:two_groups} holds for them.
Of course $\dim W = m$ and $\dim W + \dim W^{\orth} = \abs{V}$.
This proves the claim.
\end{proof}

Now we turn to the metagraph $\c{G} = (\c{V}, \c{E})$ from~\Cref{sec:graph_gps} and prove that the partition of its vertex set $\c{V}$ into $\sim_H$-equivalence classes is equitable.

\begin{proposition}
For any subgroup $H \subseteq \Sym_{n}$ the partition $\c{V} = \cup_{\overline{x} \in \c{V}_{/ H}} \overline{x}$ generated by the equivalence relation $\sim_H$ is equitable.
\end{proposition}
\begin{proof}
If $x, x' \in \overline{x}$ then $x' = \sigma(x)$ for some permutation $\sigma \in H$.
If $(x, y) \in \c{E}$ then $x$ and $y$ differ by a single edge, say $i$th one.
Obviously, then $\sigma(x)$ and $\sigma(y)$ differ by a single edge, the $\sigma(i)$th one.
Thus, of course $(x', \sigma(y)) \in \c{E}$.
The converse also holds: if $(x', \sigma(y))$ is an edge, then $(x, y)$ is an edge because $\sigma^{-1} \in H$.
It follows that $\deg(x, \overline{y}) = \deg(x', \overline{y})$ for all $\overline{y} \in \c{V}_{/ H}$.
\end{proof}

With this, we have necessary tools to prove~\Cref{thm:quotient,thm:hardness} of~\Cref{sec:equiv_gps}, starting with the former.

\subsection{$\Phi$-kernels on Quotient Graphs and the Proof of~\Cref{thm:quotient}}

Put $h = \abs[0]{\c{V}_{/ H}}$.
Let us enumerate graphs in $\c{V}$ by numbers from $1 \leq i \leq 2^d$ and equivalence classes in $\c{V}_{/ H}$ by numbers from $1 \leq j \leq h$: denote the $i$th graph by $g(i)$ and the $j$th class by $c(j)$.

We start by proving the following elementary lemma.
\begin{lemma}
Take $x, z \in \c{V}$ such that $z \in \overline{x}$.
Define
\[
\alpha(x, z)
=
\abs{\cbr{\sigma \in H : \sigma(x) = z}}
\]
Then $\alpha(x, z) = \alpha(x', z')$ for all $x', z' \in \overline{x}$.
In particular, $\alpha(x, z) = \alpha(x, x)$.
\end{lemma}
\begin{proof}
We have
\[
\alpha(z, x)
=
\abs{\cbr{\sigma \in H : \sigma(z) = x}}
=
\abs{\cbr{\sigma \in H : \sigma^{-1}(x) = z}}
=
\alpha(x, z)
\]
hence it is enough to prove that $\alpha(x, z) = \alpha(x, x)$.
Since $z \in \overline{x}$, we have $z = \sigma_z(x)$.
Write
\[
\alpha(x, z)
&=
\abs{\cbr{\sigma \in H : \sigma(x) = \sigma_z(x)}}
=
\abs{\cbr{\sigma \in H : \sigma_z^{-1}(\sigma(x)) = x}}
\\
&=
\abs{\cbr{\sigma_z^{-1} \sigma \in \sigma_z^{-1} H : \sigma_z^{-1}(\sigma(x)) = x}}
\\
&=
\abs{\cbr{\sigma_z^{-1} \sigma \in H : \sigma_z^{-1}(\sigma(x)) = x}}
=
\abs{\cbr{\sigma \in H : \sigma(x) = x}}
=
\alpha(x, x).\qedhere
\]
\end{proof}

With this, we are ready to prove that $\Pr$ from~\Cref{sec:equiv_gps} is indeed an orthonormal projector.
First, it is obvious that $\Pr \Pr f = \Pr f$ for all $f \in L^2(\c{V})$ and thus $\Pr$ is a projection.
Recall that $W$ denotes the subspace of $L^2(\c{V})$ consisting of functions that are constant on all equivalence classes.
It is easy to see that $\Pr f \in W$ for all $f \in L^2(\c{V})$ and $\Pr f = f$ if $f \in W$.
Finally, let us prove that if $f \in W^{\orth}$, i.e., if $\sum_{x \in c(j)} f(x) = 0$ for all $j = 1, \ldots, h$, then $\Pr f = 0$.
Write
\[
(\f{Pr} f)(x)
&=
\frac{1}{\abs{H}}
\sum_{\sigma \in H}
f(\sigma(x))
=
\frac{1}{\abs{H}}
\sum_{z \in \overline{x}}
f(z) \alpha(x, z)
\\
&=
\frac{1}{\abs{H}}
\sum_{z \in \overline{x}}
f(z) \alpha(x, x)
=
\frac{\alpha(x, x)}{\abs{H}}
\sum_{z \in \overline{x}}
f(z)
=
0.
\]

We will use this to verify~\Cref{thm:quotient} of~\Cref{sec:equiv_gps}.

\ThmQuotient*
\begin{proof}
Consider the adjacency matrix $\m{A}$ of $\c{G}$ as an operator $\m{A}: L^2(\c{V}) \to L^2(\c{V})$ and use~\Cref{thm:equitable_basis} to choose an orthonormal basis $\cbr{f_j}_{j=1}^{2^d}$ of eigenfunctions of $\m{A}$ such that $f_1, \ldots, f_h$ are constant on all equivalence classes and the rest average to zero over all equivalence classes.
Denote the eigenvalues of $\m{A}$ corresponding to $f_j$ by $\lambda_j$ and the eigenvalues of the symmetric normalized Laplacian $\m{\Delta}_{sym}$ by $\lambda_j^{sym} = 1 - \lambda_j/d$.
Then a $\Phi$-class kernel $k$ on $\c{G}$ is given by
\[
k(x, y)
=
\sum_{j=1}^h \Phi(\lambda_j^{sym}) f_j(x) f_j(y)
+
\sum_{j=h+1}^{2^d} \Phi(\lambda_j^{sym}) f_j(x) f_j(y).
\]

Because $\Pr$ is an orthonormal projector onto the space $W$, we have $\Pr f_j = f_j$ for $j \in \cbr{1, \ldots, h}$ and $\Pr f_j = 0$ for $j > h$.
With this, we can write
\[
k_{/ H}(x, y)
=
&\sum_{j=1}^h \Phi(\lambda_j^{sym}) (\f{Pr} f_j)(x) (\f{Pr} f_j)(y) \\
&+
\sum_{j=h+1}^{2^d} \Phi(\lambda_j^{sym}) (\f{Pr} f_j)(x) (\f{Pr} f_j)(y)
\\
=
&\sum_{j=1}^h \Phi(\lambda_j^{sym}) f_j(x) f_j(y).
\]

Let us denote the adjacency matrix of the weighted graph $\c{G}_{/H} = (\c{V}_{/ H}, \c{E}_{/ H})$ by $\m{A}_{/H}$.
Define also the $2^d \x h$ matrix $\m{S}$ with $\m{S}_{i j} = 1$ if $g(i) \in c(j)$ and $\m{S}_{i j} = 0$ otherwise.
Then we have 
\[
(\m{A}_{/H})_{i j}
=
\sum_{x \in c(i)}
\sum_{y \in c(j)}
\m{A}_{x y}
&&
\text{and thus}
&&
\m{A}_{/H} = \m{S}^{\top} \m{A} \m{S}.
\]

Computing the corresponding degree matrix yields
\[
\del{\m{D}_{/H}}_{i i}
=
\sum_{j = 1}^h \del{\m{A}_{/H}}_{i j}
=
\sum_{j = 1}^h
\sum_{x \in c(i)}
\sum_{y \in c(j)}
\m{A}_{x y}
=
\sum_{x \in c(i)}
\sum_{y \in \c{V}}
\m{A}_{x y}
=
\abs{c(i)} d.
\]

If we introduce the diagonal matrix $\m{N}$ with $\m{N}_{i i} = \abs{c(i)}$, then $\m{D}_{/ H} = d \m{N}$.
Note that we have $\m{S} \m{N}^{-1} \m{S}^{\top} f_j = f_j$.
Define $g_j = \m{N}^{-1/2} \m{S}^{\top} f_j \in L^2(\c{V}_{/ H})$ and $\m{\Delta}_{sym / H} = \m{I} - \m{D}_{/ H}^{-1/2} \m{A}_{/ H} \m{D}_{/ H}^{-1/2}$.
We thus have
\[
\m{\Delta}_{sym / H} \, g_j
&=
g_j - \m{D}_{/ H}^{-1/2} \m{S}^{\top} \m{A} \m{S} \m{D}_{/ H}^{-1/2} \m{N}^{-1/2} \m{S}^{\top} f_j
\\
&=
g_j - \m{D}_{/ H}^{-1/2} \m{S}^{\top} \m{A}  d^{-1/2} \m{S} \m{N}^{-1} \m{S}^{\top} f_j
\\
&=
g_j - \m{D}_{/ H}^{-1/2} \m{S}^{\top} \m{A} d^{-1/2} f_j
\\
&=
g_j - \m{D}_{/ H}^{-1/2} \m{S}^{\top} \lambda_j d^{-1/2} f_j
\\
&=
g_j - d^{-1} \lambda_j \m{N}^{-1/2} \m{S}^{\top} f_j
\\
&=
g_j - d^{-1} \lambda_j g_j
=
(1 - d^{-1} \lambda_j) g_j
=
\lambda_j^{sym} g_j.
\]
Hence, the $\Phi$-kernel $k_{\Phi}$ on $\c{G}_{/ H}$ corresponding to the same $\Phi$ and to the symmetric normalized Laplacian is given by
\[
k_{\Phi}(\overline{x}, \overline{y})
=
\sum_{j=1}^h \Phi(\lambda_j^{sym}) g_j(\overline{x}) g_j(\overline{y}).
\]

Denote $\psi(\overline{x}) = \abs{\overline{x}}^{-1/2}$.
It is easy to see by definition of $g_j$ that we have $f_j(x) = \psi(\overline{x}) g_j(\overline{x})$ for indices $j \in \cbr{1, \ldots, h}$.
Thus, obviously,
\[
k_{/ H}(x, y)
=
\psi(\overline{x})\psi(\overline{y})
k_{\Phi}(\overline{x}, \overline{y})
\]
which proves the claim.
\end{proof}

\subsection{The Hardness Result}

In this section we prove~\Cref{thm:hardness}.
We start with a simple lemma inspired by~\textcite{gartner2003}.

\begin{lemma} \label{thm:injective}
Consider a kernel $k: \c{V} \x \c{V} \to \R$ such that $k(\sigma_1(x), \sigma_2(y)) = k(x, y)$ for all permutations $\sigma_1, \sigma_2 \in H$.
Define the kernel $\tilde{k}: \c{V}_{/ H} \x \c{V}_{/ H} \to \R$ by $\tilde{k}(\overline{x}, \overline{y}) = k(x, y)$ and assume that $\tilde{k}(\overline{x}, \overline{y}) = \innerprod{\phi(\overline{x})}{\phi(\overline{y})}$ for a certain feature map $\phi: \c{V}_{/ H} \to \R^l$, $l \in \N$.
If $\phi$ is injective, i.e., if $\phi(\overline{x}) = \phi(\overline{y})$ implies $\overline{x} = \overline{y}$, then $k(x, y) = (k(x, x) + k(y, y)) / 2$ implies $x \sim y$.
\end{lemma}
\begin{proof}
Write
\[
k(x, x) + k(y, y) - 2 k(x, y)
&=
\tilde{k}(\overline{x}, \overline{x}) + \tilde{k}(\overline{y}, \overline{y}) - 2 \tilde{k}(\overline{x}, \overline{y})
\\
&=
\innerprod{\phi(x) - \phi(y)}{\phi(x) - \phi(y)}
\\
&=
\norm{\phi(x) - \phi(y)}^2
\]
This means that $k(x, y) = (k(x, x) + k(y, y))/2$ is equivalent to $\phi(\overline{x}) = \phi(\overline{y})$, which is by assumption is equivalent to $\overline{x} = \overline{y}$, i.e., $x \sim y$.
\end{proof}

It is therefore enough to evaluate such a kernel $k$ at three pairs of inputs to check whether $x \sim y$.
In particular, if $\sim$ is the graph isomorphism relation $\isom$ (i.e., $H = \Sym_n$), this shows that computing the kernel pointwise is at least as hard as resolving the graph isomorphism problem.

We prove~\Cref{thm:hardness} by showing that the $H$-invariant version of a $\Phi$-kernel corresponding to a strictly positive $\Phi$ may be represented via an injective feature map, a consequence of the partition $\c{V} = \cup_{\overline{x} \in \c{V}_{/ H}} \overline{x}$ of the set of graphs into sets of equivalence classes being equitable.

\ThmHardness*
\begin{proof}
From~\Cref{thm:quotient} we know that $k_{/ H} = \psi(x) \psi(y) k_{\Phi}(\overline{x}, \overline{y})$ where $k_{\Phi}$ is the $\Phi$-kernel on the quotient graph $\c{G}_{/ H}$ corresponding to the the symmetric normalized Laplacian and the same $\Phi$ as~$k$.
Recall the notation $h = \abs{\c{V}_{/ H}}$.
It follows that for an orthonormal basis $\cbr{f_j}_{j=1}^h \in L^2(\c{V}_{/ H})$ of eigenfunctions $\m{\Delta}_{sym} f_j = \lambda_j f_j$ we have
\[ \label{eqn:k_sum_hardness}
k_{/ H}(x, y)
=
\sum_{j=1}^h
\Phi(\lambda_j)
\psi(\overline{x})
f_j(\overline{x})
\psi(\overline{y})
f_j(\overline{y})
&&
\Phi(\lambda_j) > 0.
\]
Consider the function $\tilde{k}_{/ H}: \c{V}_{/ H} \x \c{V}_{/ H} \to \R$ given by $\tilde{k}_{/ H}(\overline{x}, \overline{y}) = k_{/ H}(x, y)$.
Then, by~\Cref{eqn:k_sum_hardness}, it may be represented as $\tilde{k}_{/ H}(\overline{x}, \overline{y}) = \innerprod{\phi(\overline{x})}{\phi(\overline{y})}$ with $\v{\phi}: \c{V}_{/ H} \to \R^h$ given by
\[
\v{\phi}(\overline{x})
=
\del{
\Phi(\lambda_1)^{1/2} \psi(\overline{x}) f_1(\overline{x}),
\ldots,
\Phi(\lambda_h)^{1/2} \psi(\overline{x}) f_h(\overline{x})
}^{\top}.
\]

We now prove that $\v{\phi}$ is injective.
Since the functions $f_j$ form an orthonormal basis in $L^2(\c{V}_{/ H})$, the functions $f_j'(\overline{x}) = \Phi(\lambda_j)^{1/2} \psi(\overline{x}) f_j(\overline{x})$ form a linear basis of the same space.
If we assume that $\v{\phi}(\overline{x}) = \v{\phi}(\overline{y})$ for $\overline{x} \not= \overline{y}$, then the function $\1_{\overline{x}}: \c{V}_{/ H} \to \R$ is not representable in the this basis because for all $f \in \Span \cbr{f_j'}_{j=1}^h$ our assumption implies $f(\overline{x}) = f(\overline{y})$, hence a contradiction.
Now the claim follows from~\Cref{thm:injective} and the remarks afterwards by virtue of $\v{\phi}$ being injective.
\end{proof}

\section{Additional Experimental Details} \label{appdx:experiments}

\subsection{Computational Considerations}
To ensure numerical stability during the dynamic program pre-computation
\eqref{eq:g-d-j-m-recurrence}, we use a normalizing factor of $\binom{d}{j}$ (as done in
\eqref{eq:normalized-dp-recurrence}). This factor must be accounted for in the kernel evaluation. We normalize the kernel such that $k(x, x) = 1$ to ensure that $\binom{d}{j}$ factor does not cause numerical instability during hyperparameter optimization. In code, the kernel normalization is performed in the log scale using the Log-Sum-Exp (LSE) trick. 

While some of the computations, especially those over possible hamming distances in the \textsc{Projected} case, can be made more efficient, the preliminary version of our implementation precomputes the hamming distances between possible permutations over equivalence classes in the $\textsc{Projected}$ and caches them to be reused during training and hyperparameter optimization.

\subsection{Data Preparation}
We use the \textsc{FreeSolv} dataset as provided by \texttt{Pytorch Geometric} under the
\texttt{MoleculeNet} dataset. The dataset consists of 642 molecules with experimentally determined hydration free energy values. We remove invalid molecules, and molecules with only a single atom, and follow a random train/test split of 80/20 to obtain 510 examples for training and 128 examples for testing. The limited data setting here matches the typical application settings for Gaussian processes.

To test the exact projected Gaussian processes, we consider a smaller subset of \textsc{FreeSolv} (\textsc{FreeSolv-S}). Here, we allow molecules whose constituent atoms are within the allowed set of four atom types (carbon, nitrogen, oxygen, chlorine), with a maximum of 3 atoms per atom type. Following a similar train/test split of 80/20 gives us 52 training and 13 text examples. 

For both \textsc{FreeSolv} \& \textsc{FreeSolv-S}, the data normalized by subtracting the mean and dividing by the corresponding standard deviation. The normalization is only performed for the hydration free energies. 

\subsection{Hyperparameter Tuning}
We optimize the hyperparameters of the kernel during training. For the Heat kernel in both the \textsc{Graph} and \textsc{Projected} settings, the parameters $\kappa$ and $\sigma^2$ are optimized. For the Mat\'ern kernel under similar settings, we also optimized over the additional parameter $\nu$. We experimented with initial values of $1.0$ and $2.0$ for $\kappa$ for both the Heat and Mat\'ern kernels, and observed that the training procedure converged to the same results for both situations. During model training, we also used learning rates of $0.1$ and $0.001$, with the higher learning rate improving the speed of convergence in most settings.

Another phenomenon we noticed was that the Mat\'ern kernel for typical $\nu$ values of $0.5, 1.5$ or
$2.5$, decays very quickly, especially as $n$ (and $d$) increase. The decay behavior is more
reasonable for larger values of $\nu$. In our experiments, we therefore set $\nu = \frac{d}{2} +
\nu_{base}$ where $\nu_{base} \in \{0.5, 1.5, 2.5\}$. Figure~\ref{fig:graph_kernels} shows this decaying behavior of $\nu$ for different $n$.